\documentclass[10pt]{article} % For LaTeX2e
\usepackage[preprint]{tmlr}

\usepackage{amsmath,amsfonts,bm}

\def\eqref#1{equation~\ref{#1}}
\def\1{\bm{1}}

\DeclareMathAlphabet{\mathsfit}{\encodingdefault}{\sfdefault}{m}{sl}
\SetMathAlphabet{\mathsfit}{bold}{\encodingdefault}{\sfdefault}{bx}{n}

\usepackage{cancel}
\usepackage{amsmath} 
\usepackage{amsfonts}
\DeclareMathOperator{\EX}{\mathbb{E}}% expected value
\usepackage{amsthm}
\usepackage{mathtools}
\usepackage{graphicx}
\usepackage{color}
\usepackage{subcaption}
\usepackage{graphicx}
\usepackage{booktabs}
\usepackage{multirow}
\newtheorem{example}{Example}
\usepackage{algorithmic}
\usepackage{algorithm}
\usepackage{natbib}

\newcommand{\ie}{{\em i.e., }}
\newcommand{\eg}{{\em e.g., }}

\renewcommand\vec{\boldsymbol}
\newcommand{\Z}[2]{\vec{z}_{{#1},{#2}}}
\newcommand{\Zonly}{\vec{z}}
\newcommand{\z}[2]{\vec{z}_{{#1},{#2}}}
\newcommand{\zonly}{\vec{z}}
\newcommand{\X}[2]{\vec{x}_{{#1},{#2}}}
\newcommand{\xonly}{\vec{x}}

\newcommand{\Y}[2]{\vec{y}_{{#1},{#2}}}

\newcommand{\Sk}[1][k]{\vec{S}_{#1}}
\newcommand{\sk}[1][k]{\vec{S}_{#1}}
\newcommand{\Mk}[1][k]{M_{#1}}
\newcommand{\mk}[1][k]{M_{#1}}

\newcommand{\Thetak}[1][k]{\vec{\theta}_{#1}}
\newcommand{\thetak}[1][k]{\vec{\theta}_{#1}}

\newcommand{\phik}[1][k]{\vec{\phi}_{#1}}

\newcommand{\hk}[1][k]{\vec{h}_{#1}}
\newcommand{\ck}[1][k]{\vec{c}_{#1}}

\newcommand{\norm}[1]{\| #1 \|^2}
\newcommand{\inpr}[2]{\langle #1 , #2 \rangle}
\theoremstyle{plain}
\newtheorem{theorem}{Theorem}[section]

\newtheorem{lemma}[theorem]{Lemma}

\theoremstyle{definition}

\theoremstyle{remark}
\newtheorem{remark}[theorem]{Remark}

\usepackage{hyperref}
\usepackage{url}

\title{Differentiated Aggregation to Improve Generalization in Federated Learning\\[0.6em]
{\large\mdseries\sffamily Published in Transactions on Machine Learning Research 2025}}

\author{\name Peyman Gholami \email pghola2@uic.edu \\
      \addr Department of Electrical and Computer Engineering\\
      University of Illinois Chicago\\
      \\
      \name Hulya Seferoglu \email hulya@uic.edu \\
      \addr Department of Electrical and Computer Engineering\\
      University of Illinois Chicago}

\def\month{10}  % Camera-ready month (used with [accepted])
\def\year{2025} % Camera-ready year (used with [accepted])
\def\openreview{\url{https://openreview.net/forum?id=F5hgpQ1Ccd}}

\begin{document}

\maketitle

\begin{abstract}

% The Federated Averaging (FedAvg) algorithm is a widely utilized technique in Federated Learning. It follows a recursive pattern where nodes perform a few local stochastic gradient descents (SGD), and then the central server updates the model by taking an average. The primary purpose of conducting model averaging is to mitigate the consensus error that arises between models across different nodes. 
% In our empirical examination, it becomes evident that in non-iid data distribution setting, the consensus error in the initial layers of a deep neural network is considerably smaller than that observed in the final layers. This observation hints at the feasibility of applying a less intensive averaging approach for the initial layers. Typically, these layers are designed to extract meaningful representations from the neural network's input. To delve deeper into this phenomenon, we formally analyze it within the context of linear representation.
% We illustrate that increasing the number of local SGD iterations or reducing the frequency of averaging for the representation extractor leads to enhanced generalizability in the learned model produced by FedAvg's output.
% The paper is followed with experimental results showing the effectiveness of this method.

%%%%%%

This paper focuses on reducing the communication cost of federated learning by exploring generalization bounds and representation learning. We first characterize a tighter generalization bound for one-round federated learning based on local clients' generalizations and heterogeneity of data distribution (non-iid scenario). We also characterize a generalization bound in R-round federated learning and its relation to the number of local updates (local stochastic gradient descents (SGDs)). Then, based on our generalization bound analysis and its interpretation through representation learning, we infer that less frequent aggregations for the representation extractor (typically corresponds to initial layers) compared to the head (usually the final layers) leads to the creation of more generalizable models, particularly in non-iid scenarios. We design a novel Federated Learning with Adaptive Local Steps (FedALS) algorithm based on our generalization bound and representation learning analysis.
FedALS employs varying aggregation frequencies for different parts of the model, so reduces the communication cost. The paper is followed with experimental results showing the effectiveness of FedALS. Our \href{https://github.com/Peyman-gholami/FedALS}{codes} are available for reproducibility.

%by answering the following question: When and which parts of local models shall be aggregated. We attempt to address this question by introducing (i) improved generalization error bounds for federated learning, (ii) interpreting the generalization bounds through representation learning, and (iii) designing a novel federated learning model FedALS. 

\end{abstract}
\vspace{-7pt}
\section{Introduction} \label{sec1}
\vspace{-2pt}

Federated learning advocates that multiple clients collaboratively train machine learning models under the coordination of a parameter server (central aggregator) \citep{McMahan2016CommunicationEfficientLO} by leveraging all clients' computational capacities. 
%This approach has great potential for preserving the privacy of data stored by clients while leveraging all clients' computational capacities. 
Despite its promise, federated learning suffers from high communication costs between clients and the parameter server. 
%
%In a federated learning setup, a parameter server oversees a global model and distributes it to participating clients. These clients then conduct local training using their own data. Then, the clients send their model updates to the parameter server, which aggregates them to a global model. This process continues until convergence. 
%
In particular, exchanging machine learning models between clients and the parameter server is costly, especially for large models, which are typical in today's machine learning applications \citep{Konecn2016FederatedLS,Zhang2013InformationtheoreticLB,BarnesLowerBounf,Braverman2015CommunicationLB, zibaeirad2024comprehensivesurveysecuritysmart}. Furthermore, the uplink bandwidth of clients may be limited, time-varying and expensive. Thus, there is an increasing interest in reducing the communication cost of federated learning via (i) multiple local updates, also known as ``Local SGD'' \citep{Stich2018LocalSC,Stich2019TheEF,Wang2018unified}, (ii) pruning \cite{wangni2017gradientsparsificationcommunicationefficientdistributed, jiang2022modelpruningenablesefficient}, (iii) quantization {\cite{quant1,quant2}, etc. In this paper, we take a dramatically different and complementary approach:  \emph{Aggregating the different parts of a machine learning model at different frequencies in a federated learning setup.}

The primary purpose of communication in federated learning is to periodically aggregate local models to reduce the consensus distance among clients. This practice helps maintain the overall optimization process on a trajectory toward global optimization. It is important to note that when the consensus distance among clients becomes substantial, the convergence rate reduces.
This occurs as individual clients gradually veer towards their respective local optima without being synchronized with the models from other clients.
This issue is amplified when the data distribution among clients is non-iid.
It has been demonstrated that the consensus distance is correlated to (i) the randomness in each client's own dataset, which causes variation in consecutive local gradients, as well as (ii) the dissimilarity in loss functions among clients due to non-iidness \citep{Stich2019TheEF,gholami2023digest}.
%
%to the variability within nodes' gradients of local loss functions caused by the randomness in their individual datasets, as well as the dissimilarity in loss functions between nodes caused by non-iidness \citep{Stich2019TheEF,gholami2023digest}.
%
More specifically, the consensus distance at iteration $t$ is defined as $\frac{1}{K}\sum_{k=1}^K \norm{\hat{{\thetak[t]}}  -\thetak[k,t]}$, where $\hat{{\thetak[t]}} = \frac{1}{K}\sum_{k=1}^K \thetak[k,t]$, $K$ is the number of clients, $\thetak[k,t]$ is the local model at client $k$ at iteration $t$, and $\norm{\cdot}$ is squared $l_2$ norm.
Note that the consensus distance goes to zero when global aggregation is performed at each communication round. This makes the communication of models between clients and the parameter server crucial, but this introduces significant communication overhead.  This paper aims to reduce federated learning communication overhead through the following contributions.

\textit{\underline{Contribution I:} Improved Generalization Error Bound.} The generalization error of a learning model is defined as the difference between the model's empirical risk and population risks. (We provide a mathematical definition in Section \ref{sec:back}).
% If a model has a small generalization error, it is considered that the model enjoys effective generalization.
Existing approaches for training models mostly minimize the empirical risk or its variants. However, a small population risk is desired, showing how well the model performs in the test phase as it denotes the loss that occurs when new samples are randomly drawn from the distribution.
Note that a small empirical risk and a reduced generalization error correspond to a low population risk.
Thus, there is an increasing interest in establishing an upper limit for the generalization error and understanding the underlying factors that affect the generalization error. The generalization error analysis is also important to quantitatively assess the generalization characteristics of trained models, provide reliable guarantees concerning their anticipated performance quality, and design new models and systems. 

In this paper, we offer a tighter generalization bound compared to the state of the art \cite{Barnes2022ImprovedIG,9154277,sun2023understanding} for federated learning, considering local clients' generalizations and non-iidness (i.e., heterogeneous data distribution across the clients).
% \textcolor{blue}{We have incorporated the concept of representation learning into the derived bounds, which will form the basis of FedALS.}

%In this paper, we 

%Thus, communication of models between clients and the parameter server is crucial to reduce or eliminate the consensus error.  
%every iteration that we do global averaging this quantity becomes zero over the network.

\iffalse
They are solving a personalized federated learning problem in which the head of the model is learned locally with no averaging at all and the representation extractor is learned using federated learning.
They do not discuss how having a large number of local steps helps the representation extractor generalization to be improved and there is no notation of generalization error bound in their analysis.
\fi

\textit{\underline{Contribution II:} Representation Learning Interpretation. }
Recent studies have demonstrated that the concept of representation learning is a promising approach to reducing the communication cost of federated learning \citep{Collins2021ExploitingSR}. 
%
%improve the performance of FL. 
%
This is achieved by leveraging the shared representations in all clients' datasets. For example, let us consider a federated learning application for image classification, where different clients have datasets of different animals.  Despite each client having a different dataset (one client has dog images, another has cat images, etc.), these images usually have common features such as an eye/ear shape.   
%
%This occurs because individual nodes' data often share a common set of predictive features for labels, despite variations in these datasets among nodes. For instance, consider the heterogeneous federated image classification scenario, where different nodes possess images of different animals. It is reasonable to assume that these images share key features, such as eye or ear shape, that allow for a straightforward and accurate mapping from the feature space to the label space. 
%
These shared features, typically extracted in the same way for different types of animals, require consistent layers of a neural network to extract them, whether the animal is a dog or a cat. As a result, these layers demonstrate similarity (i.e., less variation) across clients even when the datasets are non-iid. 
%are likely to exhibit minimal variation among nodes, even when the datasets are non-iid. 
This implies that the consensus distance for this part of the model (feature extraction) is likely smaller. Based on these observations, our key idea is to reduce the aggregation frequency of the layers that show high similarity, where these layers are updated locally between consecutive aggregations. This approach would reduce the communication cost of federated learning as some layers are aggregated, hence their parameters are exchanged, less frequently.
% This makes it crucial to determine the layers that show high similarity.
The next example scratches the surface of the problem for a toy example. 

%shed a light on this problem via a toy example. 

%We provide an example next to illustrate this idea.  

%while the frequency of aggregation of other layers could be still high. Since reducing the frequency of aggregation will reduce the communication 

%making it feasible to reduce the frequency of aggregation, which reduces the communication overhead in federated learning.

\begin{example}
We consider a federated learning setup of five clients with a central parameter server to train a ResNet-$20$ \citep{He2015resnet} on a heterogeneous partition of CIFAR-$10$ dataset \citep{cifar10}.
We use Federated Averaging (FedAvg) \citep{McMahan2016CommunicationEfficientLO} as an aggregation algorithm since it is the dominant algorithm in federated learning.
%This approach computes the aggregated model by averaging the models from all nodes.
%
We applied FedAvg with $50$ local steps prior to each averaging step, denoted as $\tau = 50$.
Non-iidness is introduced by allocating $2$ classes to each client.
Finally, we evaluate the quantity of the average consensus distance for each model layer during the optimization in Fig. \ref{consensus}.
It is clear that the initial layers have smaller consensus distance as compared to the final layers. 
%a reduced level of consensus distance when contrasted with the final layers. 
This is due to initial layers' role in extracting representations from input data and their higher similarity across clients.
\end{example}

\begin{figure}[tb]
\centering
\includegraphics[width=0.4\textwidth]{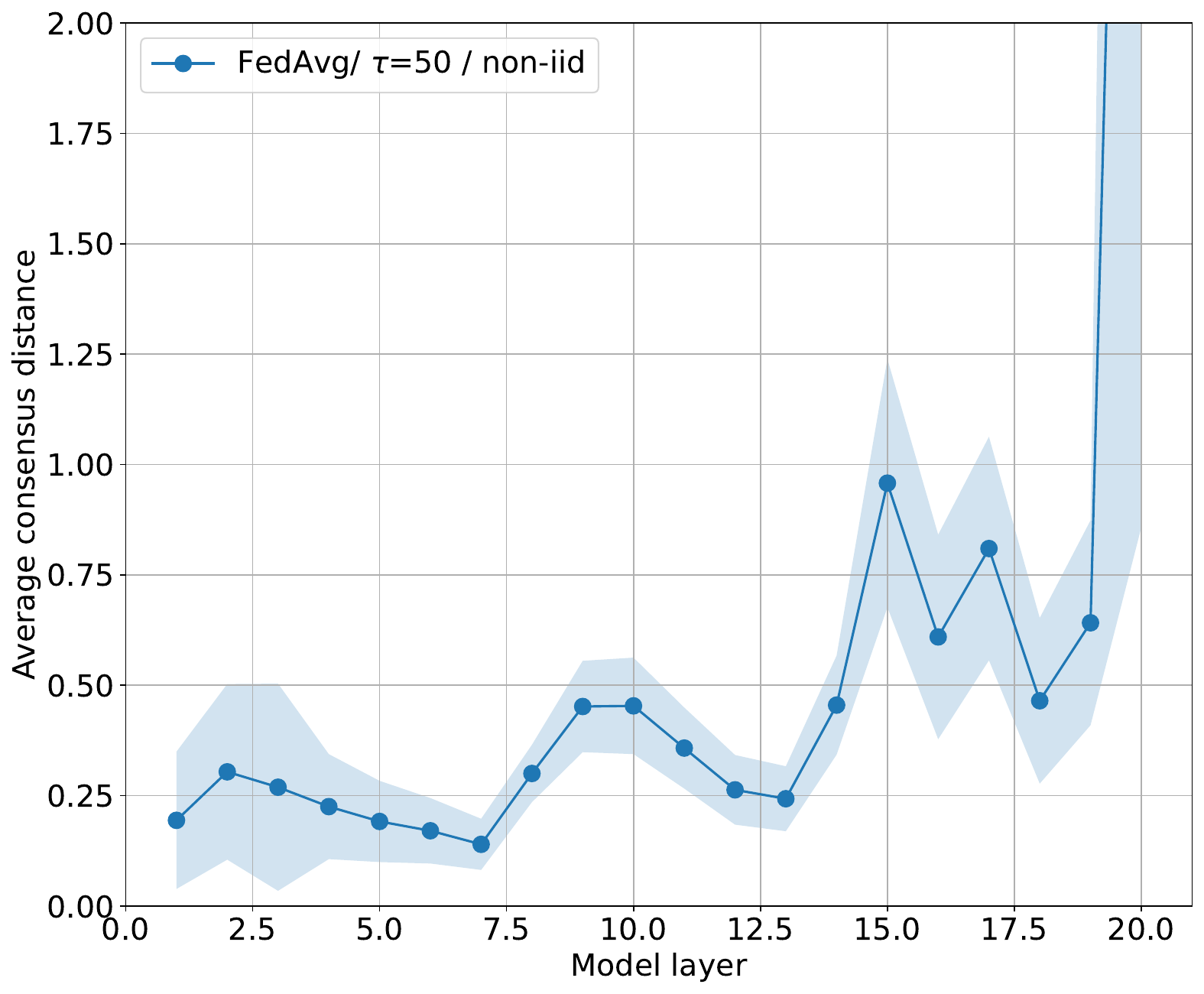} % Reduce the figure size so that it is slightly narrower than the column.
\vspace{-7pt}
\caption{Average consensus distance over time for different layers during ResNet-$20$ training with FedAvg on CIFAR-$10$ ($5$ clients, non-iid, $2$ classes per client) shows that early layers have lower consensus distance.}
\vspace{-15pt}
\label{consensus}
\end{figure}

The above example indicates that initial layers show higher similarity, so they can be aggregated less frequently.
Additionally, several empirical studies \citep{reddi2021adaptive,Yu2020SalvagingFL} show that federated learning with multiple local updates per round learns a generalizable representation and is unexpectedly successful in non-iid settings. 
{
This view is supported by recent theoretical work showing that as you go deeper into a network, the way it represents information changes in a structured way. For example, \citet{Li_2021} and \citet{Zavatone_Veth_2022} have found that a layer's representation is a blend of the input data's structure and the learning task's structure. Critically, they showed that the influence of the task becomes stronger in deeper layers, meaning early layers focus more on general input features while later layers adapt to focus on what's needed to solve the task. Also \cite{bordelon2022selfconsistentdynamicalfieldtheory}, by doing a study on learning dynamics, confirmed this, showing that deeper layers change their representations more significantly during training to ultimately align with the final goal. Taken together, these papers provide a strong theoretical basis for the idea that networks learn by gradually shifting their focus from representing the input data in early layers to representing the task solution in later layers.
}
These studies encourage us to delve deeper into investigating how local updates and model aggregation frequency for the model's representation extractor and the head affect model generalization. 
% In this paper, we aim to analytically show that the conclusion of Example \ref{consensus} holds for a general setup with any number of clients and for heterogeneous (non-iid) data distribution.
% we aim to analytically show that using more local steps for the representation extractor improves the performance in the general setup with any number of clients and for heterogeneous (non-iid) data distribution.
% We start with analyzing generalization error bounds because if a model has a small generalization error, it is considered that it enjoys effective generalization, leading to a better trained model.

In this paper, based on our improved generalization bound analysis and our representation learning interpretation of this analysis, we showed for the first time that employing different frequencies of aggregation, \ie the number of local updates (local SGDs), for the representation extractor (typically corresponding to initial layers) and the head (final label prediction layers), leads to the creation of more generalizable models particularly in non-iid scenarios.

%less frequent aggregations, hence more local updates (local SGDs), for the representation extractor (usually corresponds to initial layers) leads to the creation of more generalizable models particularly in non-iid scenarios. 

% our objective is to provide analytical evidence demonstrating that increasing the number of local steps for the representation extractor leads to enhanced performance in the general setup, irrespective of the number of clients and in scenarios involving heterogeneous (non-iid) data distributions.

% We observe from our generalization error bound analysis that  less frequent aggregations, hence more local updates (local SGDs), for the representation extractor (usually corresponds to initial layers) leads to the creation of more generalizable models particularly in non-iid scenarios.     

\textit{\underline{Contribution III:} Design of FedALS.} We design a novel Federated Learning with Adaptive Local Steps (FedALS) algorithm based on our generalization error bound analysis and its representation learning interpretation.
FedALS employs varying aggregation frequencies for different parts of the model.

\textit{\underline{Contribution IV:} Evaluation.}  We evaluate the performance of FedALS using deep neural network model ResNet-$20$ for CIFAR-$10$, CIFAR-$100$ \citep{cifar10}, SVHN \citep{Netzer2011ReadingDI}, and MNIST \citep{mnist} datasets. We also estimate the impact of FedALS on large language models (LLMs) in fine-tuning OPT-$125$M \citep{OPT} on the Multi-Genre Natural Language Inference (MultiNLI) corpus \citep{multi-nli}. We consider both iid and non-iid data distributions.
Experimental results confirm that FedALS outperforms the baselines in terms of accuracy in non-iid setups while also saving on communication costs across all setups.

\vspace{-2pt}
\section{Related Work}
\vspace{-2pt}

There has been increasing interest in distributed learning recently, largely driven by Federated Learning. Several studies have highlighted that these algorithms achieve convergence to a global optimum or a stationary point of the overall objective, particularly in convex or non-convex scenarios \citep{Stich2019TheEF,Stich2018LocalSC,gholami2023digest,Lian2018Async,Kairouz2019openFL}.
However, it is widely accepted that communication cost is the major bottleneck for these techniques in large-scale optimization applications \citep{Konecn2016FederatedLS,Lin2017DeepGC}.
To tackle this issue, two primary strategies are put forth: the utilization of mini-batch parallel SGD, and the adoption of Local SGD. These approaches aim to enhance the equilibrium between computation and communication.
\citet{Woodworth2020IsLS,Woodworth2020MinibatchVL} attempt to theoretically capture the distinction to comprehend under what circumstances Local SGD outperforms minibatch SGD.

Local SGD appears to be more intuitive compared to minibatch SGD, as it ensures progress towards the optimum even in cases where workers are not communicating and employing a mini-batch size that is too large may lead to a decrease in performance \citep{Lin2017DeepGC}.
However, due to the fact that individual gradients for each worker are computed at distinct instances, this technique brings about residual errors. As a result, a compromise arises between reducing communication rounds and introducing supplementary errors into the gradient estimations.
This becomes increasingly significant when data is unevenly distributed across nodes.
There are several decentralized algorithms that have been shown to mitigate heterogeneity \citep{Karimireddy2019SCAFFOLDSC,Liu2023DecentralizedGT}. One prominent example is the Stochastic Controlled Averaging algorithm (SCAFFOLD) \citep{Karimireddy2019SCAFFOLDSC}, which addresses the node drift caused by non-iid characteristics of data distribution. They establish the notion that SCAFFOLD demonstrates a convergence rate at least equivalent to SGD, ensuring convergence even when dealing with highly non-iid datasets.

However, despite these factors, multiple investigations \citep{reddi2021adaptive,Yu2020SalvagingFL,Lin2020Don't,gu2023why}, have noted that the model trained using FedAvg and incorporating multiple Local SGD per round exhibits unexpected effectiveness when subsequently fine-tuned for individual clients in non-iid Federated learning setting. This implies that the utilization of FedAvg with several local updates proves effective in acquiring a valuable data representation, which can later be employed on each node for downstream tasks.
Following this line of reasoning, our justification will be based on the argument that the Local SGD component of FedAvg contributes to improving performance in heterogeneous scenarios by facilitating the acquisition of models with enhanced generalizability.

An essential characteristic of machine learning systems is their capacity to extend their performance to novel and unseen data. This capacity, referred to as generalization, can be expressed within the framework of statistical learning theory. There has been a line of research to characterize generalization bound in FL \cite{wang2023generalization,agnistic}. More recently \citet{Barnes2022ImprovedIG,sun2023understanding,9154277, sefidgaran2024lessonsgeneralizationerroranalysis} considered this problem and gave upper bounds on the expected generalization error for FL in iid setting in terms of the local generalizations of clients. These studies demonstrate a $\frac{1}{K}$ dependency on the number of nodes. Motivated by this, our research focuses on analyzing generalization that improves the dependency to $\frac{1}{K^2}$ in iid setting. We then use the derived insights to introduce FedALS, aiming to enhance generalization in non-iid setup.

Our work differentiates itself from personalized federated learning work \cite{Collins2021ExploitingSR} in a way that our work focuses on a general federated learning setup, where a global model is trained collectively while personalized federated learning advocates training a model from each client's perspective. Thus, the results are strikingly different; They show that the heads should be trained locally, while our work FedALS shows that heads should be aggregated more frequently. They also consider only linear models, while our work is generic.
\citet{collins2022fedavgfinetuninglocal} analyze the behavior of FedAvg in multi-task linear regression with a common representation, focusing on the convergence behavior of the algorithm. They prove a faster convergence rate when more than one local step is deployed per round. However, they do not address generalization in the context of statistical learning, which is a crucial aspect of our work.

\vspace{-3pt}
\section{Background and Problem Statement \label{sec:back}}
%\section{Theoretical Analysis}
\vspace{-2pt}

\subsection{Preliminaries and Notation}
%\subsection{Technical Preliminaries}
\vspace{-2pt}
We consider that we have $K$ clients/nodes in our system, and each node has its own portion of the dataset. For example, node $k$ has a local dataset $\Sk = \{\Z{k}{1},...,\Z{k}{n_k}\}$, where $\Z{k}{i} = (\X{k}{i},\Y{k}{i})$ is drawn from a distribution $\mathcal{D}_k$ over $\mathcal{X} \times \mathcal{Y}$ , where $\mathcal{X}$ is the input space
and $\mathcal{Y}$ is the label space.
We consider $\mathcal{X} \subseteq \mathbb{R}^d$  and $\mathcal{Y} \subseteq \mathbb{R}$ .
The size of the local dataset at node $k$ is $n_k$. The dataset across all nodes is defined as $\Sk[] = \{\Sk[1], ..., \Sk[K]\}$. Data distribution across the nodes could be independent and identically distributed (iid) or non-iid. In iid setting, we assume that $\mathcal{D}_1 = ... = \mathcal{D}_K = \mathcal{D}$ holds. On the other hand, non-iid setting covers all cases where this equality does not hold. 
%\textcolor{blue}{Throughout the paper, we discuss the distribution of data among nodes to be independent and identically distributed (iid) or non-iid. The former assumes that $\mathcal{D}_1 = ... = \mathcal{D}_K = \mathcal{D}$, while the latter describes situations where these equations do not apply.}

We assume that $\Mk[{\Thetak[]}] = \mathcal{A}(\Sk[])$ represents the output of a possibly stochastic function denoted as $\mathcal{A}(\Sk[])$, where $\Mk[{\Thetak[]}]: \mathcal{X} \rightarrow \mathcal{Y}$ represents the learned model parameterized by $\Thetak[]$. We consider a real-valued loss function denoted as $l(\mk[{\thetak[]}], \zonly)$, which assesses the model $\mk[{\thetak[]}]$ based on a sample $\zonly$. 

\vspace{-5pt}
\subsection{Generalization Error}
\vspace{-2pt}
We first define an empirical risk on dataset $\sk[]$ as
\begin{align}
    \hspace{-6pt}R_{\sk[]}(\mk[{\thetak[]}]) \hspace{-2pt}=  \hspace{-2pt} \EX_{k \sim \mathcal{K}} \hspace{-2pt} R_{\sk}\hspace{-2pt}(\mk[{\thetak[]}])\hspace{-2pt} = \hspace{-2pt}\EX_{k \sim \mathcal{K}}\hspace{-2pt}\frac{1}{n_k} \hspace{-2pt} \sum_{i=1}^{n_k} l(\mk[{\thetak[]}],\z{k}{i}),\hspace{-2pt}
\end{align}
where $\mathcal{K}$ is an arbitrary distribution over nodes to weight different local risk contributions in the global risk. Specifically, $\mathcal{K}(k)$ represents the contribution of node $k$'s loss in the global loss. In the most conventional case, it is usually assumed to be uniform across all nodes, i.e., $\mathcal{K}(k)=\frac{1}{K}$ for all $k$.
%
%
% In light of these notations, we can establish the following definition:
% \begin{align}
%     R_{\sk[]}(\mk[{\thetak[]}]) &=  \EX_{k \sim \mathcal{K}} R_{\sk}(\mk[{\thetak[]}]) \\ &=\EX_{k \sim \mathcal{K}} \frac{1}{n_k} \sum_{i=1}^{n_k} l(\mk[{\thetak[]}],\z{k}{i}),
% \end{align}
%  as the empirical risk on dataset $\sk[]$. $\mathcal{K}$ is an arbitrary distribution on nodes. 
% 
$R_{\sk}(\mk[{\thetak[]}])$ is the empirical risk for model $\mk[{\thetak[]}]$ on local dataset $\sk$. 
 We further define  a population risk for model $\mk[{\thetak[]}]$ as
\begin{align}
     R(\mk[{\thetak[]}]) &= \EX_{k \sim \mathcal{K}} R_k(\mk[{\thetak[]}]) = \EX_{k \sim \mathcal{K},\Zonly \sim \mathcal{D}_k} l(\mk[{\thetak[]}],\Zonly),
 \end{align} where $R_k(\mk[{\thetak[]}])$ is the population risk on node $k$'s data distribution.
 
 Now, we can define the generalization error for dataset $\sk[]$ and function $\mathcal{A}(\sk[])$ as
 \begin{align}
     \Delta _{\mathcal{A}}(\sk[]) = R(\mathcal{A}(\sk[])) - R_{\sk[]}(\mathcal{A}(\sk[])).
 \end{align}
 The expected generalization error is expressed as $
     \EX_{\Sk[]} \Delta _{\mathcal{A}}(\Sk[]) $,
where $\EX_{\Sk[]}[\cdot] = \EX_{\{\Sk[k] \sim \mathcal{D}_k^{n_k}\}_{k=1}^K}[\cdot]$ is used for the sake of notation convenience.

% If a model has a small generalization error, it is considered that the model enjoys effective generalization. Existing approaches for training models usually minimize the empirical risk or its variants. 
% Consequently, a small empirical risk and a reduced generalization error correspond to a low population risk.
% Small population risk shows how good the model is in the test phase as it denotes the loss that occurs when new samples are randomly drawn from the distribution.
% Thus, there is an increasing interest in establishing an upper limit for the generalization error and understanding the underlying factors that affect the generalization error. The generalization error analysis is also important to quantitatively assess the generalization characteristics of trained models, provide reliable guarantees concerning their anticipated performance quality, and design new models and systems. 

% Our \underline{goal} in this paper is to \emph{use generalization error bounds to design communication efficient federated learning when data is distributed heterogeneously over nodes, \ie non-iid setup}. 

%exists a substantial interest in establishing an upper limit for the generalization error and comprehending the underlying factors that govern it. 

%These efforts enable us to quantitatively assess the generalization characteristics inherent to a machine learning system and provide reliable guarantees concerning its anticipated performance quality.

% \vspace{-7pt}
\subsection{Federated Learning}
\vspace{-2pt}
We consider a federated learning scenario with $K$ nodes/clients and a centralized parameter server. The nodes update their localized models to minimize their empirical risk $R_{\sk}(\mk[{\thetak[]}])$  on local dataset $\sk$, while the parameter server aggregates the local models to minimize the empirical risk $R_{\sk[]}(\mk[{\thetak[]}])$. 
%
%In a federated scenario comprising $K$ nodes alongside a central server, the objective is for the server to utilize the entirety of the data from various clients to develop models that yield minimal empirical risk, denoted as $R_{\sk[]}(\mk[{\thetak[]}])$. 
%
Due to connectivity and privacy constraints, the clients do not exchange their data with each other. % which demands a centralized approach.
One of the most widely used federated learning algorithms is FedAvg \citep{McMahan2016CommunicationEfficientLO}, which we explain in detail next. 

%Due to constraints related to communication and privacy, the clients are unable to exchange their individual data, thus demanding a federated approach as a solution.

At round $r$ of FedAvg, each node $k$ trains its model $\mk[{\thetak[{k,r}]}] = \mathcal{A}_{k,r}(\sk)$ locally using the function/algorithm $\mathcal{A}_{k,r}$. The local models $\mk[{\thetak[{k,r}]}]$ are transmitted to the central parameter server, which merges the received local models to aggregated model parameters $\hat{{\thetak[]}}_{r+1} = \hat{\mathcal{A}}(\thetak[{1,r}], ..., \thetak[{K,r}])$, where $\hat{\mathcal{A}}$ is the aggregation function. 
%
%local model training denoted as $\mk[{\thetak[k]}] = \mathcal{A}_k(\sk)$ using the algorithm $\mathcal{A}_k$. 
%Following individual local model training, the resultant models $\mk[{\thetak[k]}]$ are transmitted to the central parameter server. Here, they are merged to form an aggregated model denoted as $\mk[\hat{{\thetak[]}}]$, a process accomplished through an aggregation algorithm $\hat{{\thetak[]}} = \hat{\mathcal{A}}(\thetak[1], ..., \thetak[K])$. 
%
In FedAvg, the aggregation function calculates an average, so the aggregated model is expressed as \vspace{-5pt}
\begin{align}
\hat{{\thetak[]}}_{r+1} = \EX_{k \sim \mathcal{K}} \thetak[{k,r}]. \label{fedavg}
\end{align}

Subsequently, the aggregated model is transmitted to all nodes. This process continues for $R$ rounds. The final model after $R$ rounds of FedAvg is $\mathcal{A}(S)$.

The local models are usually trained using stochastic gradient descent (SGD) at each node. To reduce the communication cost needed between the nodes and the parameter server, each node executes multiple SGD steps using its local data after receiving an aggregated model from the parameter server. To be precise, 
%
%In general, the model is learned through local execution of stochastic gradient descent (SGD) on each node. To improve the communication efficiency, each node who receives the current aggregated model $\mk[\hat{{\thetak[]}}]$ in every round, executes multiple SGD steps as $\mathcal{A}_k(S_k)$ using its local data starting from $\mk[\hat{{\thetak[]}}]$.
%To make the notation more solid, 
%
we have the aggregated model parameters at round $r$ as $\hat{{\thetak[]}}_r$. Specifically, upon receiving $\hat{{\thetak[]}}_r$, node $k$ computes
\begin{align}
    \thetak[{k,r,t+1}] = \thetak[{k,r,t}] - \frac{\eta}{|\mathcal{B}_{k,r,t}|} \sum_{i \in \mathcal{B}_{k,r,t}}\nabla l(\mk[{\thetak[{k,r,t}]}],\z{k}{i})\label{local_sgd} 
\end{align}
for $t=0,\dots,\tau - 1$, where $\tau$ is the number of local SGD steps, 
$\thetak[{k,r,0}]$ is defined as $\thetak[{k,r,0}] = \hat{{\thetak[]}}_r$,
$\eta$ is the learning rate, $\mathcal{B}_{k,r,t}$ is the batch of samples used in local step $t$ of round $r$ in node $k$, $\nabla$ is the gradient, and $|\cdot|$ shows the size of a set.
Upon completing the local steps in round $r$, each node transmits $\thetak[{k,r}] = \thetak[{k,r,\tau}]$ to the parameter server to calculate $\hat{{\thetak[]}}_{r+1}$ as in (\ref{fedavg}).

% \vspace{-6pt}
\subsection{Representation Learning}
\vspace{-2pt}
Our approach for analyzing the generalization error bounds for federated learning, specifically focusing on FedAvg, uses representation learning, which we explain next.

We consider a class of models that consist of a representation extractor (\eg ResNet). Let $\thetak[]$ be the model $\mk[{\thetak[]}]$'s parameters. We can decompose $\thetak[]$ into two sets: $\phik[]$ containing the representation extractor's parameters and $\hk[]$ containing the head parameters, \ie $\thetak[] = [\phik[],\hk[]]$. $\mk[{\phik[]}]$ is a function that maps from the original input space to some feature space, \ie $\mk[{\phik[]}]: \mathbb{R}^d \rightarrow \mathbb{R}^{d'}$, where usually  $d' \ll d$. The function $\mk[{\hk[]}]$ performs a low complexity mapping from the representation space to the label space, which can be expressed as $\mk[{\hk[]}]: \mathbb{R}^{d'} \rightarrow \mathbb{R}$.

% parameterized by $\phik[]$ and a head parameterized by $\hk[]$. 
% Let $\thetak[]$ be the model $\mk[{\thetak[]}]$'s parameters, $\thetak[]$ is decomposed into two sets: $\phik[]$ containing the representation extractor's parameters, and $\hk[]$ containing the head parameters, \ie $\thetak[] = [\phik[],\hk[]]$.

For any $\xonly \in \mathcal{X}$, the output of the model is $\mk[{\thetak[]}](\xonly) = (\mk[{\hk[]}] \circ \mk[{\phik[]}])(\xonly) = \mk[{\hk[]}](\mk[{\phik[]}](\xonly))$. For instance, if $\mk[{\thetak[]}]$ is a neural network, $\mk[{\phik[]}]$ represents several initial layers of the network, which are typically designed to extract
meaningful representations from the neural network's input. On the other hand, $\mk[{\hk[]}]$ denotes the final few layers that lead to the network's output.

% \textcolor{blue}{We define $\mathcal{A}^{\phik[]^*}(\sk[], \mk[{\hk[]}])$ as the model in which $\phik[]$ component is optimized to minimize the empirical risk on dataset $\sk[]$, while the parameters of $\hk[]$ remain constant as in $\mk[{\hk[]}]$. Similarly, we define $\mathcal{A}^{\hk[]^*}(\sk[], \mk[{\phik[]}])$. Additionally, we define $\mathrm{H}(\mk[{\thetak[]}]) = \mk[{\hk[]}]$ and $\mathrm{\Phi}(\mk[{\thetak[]}]) = \mk[{\phik[]}]$, where $\mk[{\thetak[]}] = \mk[{\hk[]}] \circ \mk[{\phik[]}]$.
% }
\vspace{-5pt}
\section{Improved Generalization Bounds }
\vspace{-2pt}

In this section, we derive generalization bounds for FedAvg based on clients' local generalization performances in a general non-iid setting for the first time in the literature. First, we start with one-round FedAvg and analyze its generalization bound. Then, we extend our analysis to $R-$round FedAvg. 

% \vspace{-7pt}
\subsection{One-Round Generalization Bound}
% \vspace{-2pt}
% we assume that there is a ground-truth optimal representation function parametrized by $\phik[]^*$ and head parameterized by $\hk[]^*$  for all the tasks such that for each task t ∈ [T + 1], we have E(x,y)∼μt [y|x] = ⟨wt∗, φ∗(x)⟩.
% We consider the standard algorithmic representation learning analysis, where the representation is a linear map from the original input space to a low-dimensional space \citep{Du2020FewShotLV,Collins2021ExploitingSR,Finetune}. Namely, the representation function class is defined as $\mk[{\phik[]}] = \{\xonly \rightarrow  \xonly B \mid B \in \mathbb{R}^{d\times d'} \}$.
% This is followed by a linear head, \ie $\mk[{\hk[]}] = \{\xonly \rightarrow \xonly\wk[]\mid \wk[] \in \mathbb{R}^{d'}\}$.

In the following theorem, we determine the generalization bound for one round of FedAvg.
\begin{theorem}\label{one-round_lemma}
    Let $l(\Mk[{\Thetak[]}],\Zonly)$ be $\mu$-strongly convex and $L$-smooth in $\Mk[{\Thetak[]}]$. $\Mk[{\Thetak[]}_k]=\mathcal{A}_k(\Sk)$ represents the model obtained from Empirical Risk Minimization (ERM) algorithm on local dataset $\Sk$, \ie $\Mk[{\Thetak[]}_k]= \arg \min_{\Mk[]} \sum_{i=1}^{n_k} l(\Mk[],\z{k}{i})$, and  $\Mk[\hat{{\Thetak[]}}]=\mathcal{A}(\Sk[])$ is the model after one round of FedAvg ($\hat{{\Thetak[]}} = \EX_{k \sim \mathcal{K}} \Thetak$).
    Then, the expected generalization error, $\EX_{\Sk[]}\Delta_{\mathcal{A}}(\Sk[])$, is upper bounded with

    {\fontsize{9.5}{8}\selectfont
    \begin{align}
    \EX_{k \sim \mathcal{K}}\bigg[\frac{L\mathcal{K}(k)^2}{\mu}\underbrace{\EX_{\Sk[k]} \Delta_{\mathcal{A}_k}(\Sk)}_\text{Expected local generalization} \label{one_round-main}
    + 2\sqrt{\frac{L}{\mu}}\mathcal{K}(k)\bigg(\underbrace{\EX_{\Sk[]} \delta_{k,\mathcal{A}}(\Sk[])}_\text{Expected non-iidness} \underbrace{\EX_{\Sk[k]}\Delta_{\mathcal{A}_k}(\Sk)}_\text{Expected local generalization}\bigg)^{\frac{1}{2}}\bigg],
    \end{align}
    }
    where $\delta_{k,\mathcal{A}}(\Sk[]) =  R_{\sk}(\mathcal{A}(\Sk[])) - R_{\sk}(\mathcal{A}_k(\Sk)) $ shows the level of non-iidness at client $k$ for function $\mathcal{A}$ on dataset $\Sk[]$.
\end{theorem}
% \footnote{We note that this theorem assumes that all In the case of partial client participation in FedAvg, we propose two methods:
% \begin{enumerate}
%     \item Sampling $\hat{K}$ clients with replacement based on distribution $\mathcal{K}$, followed by averaging the local models with equal weights.
%     \item Sampling $\hat{K}$ clients without replacement uniformly at random, then performing weighted averaging of local models. Here, the weight of client $k$ is rescaled to $\frac{\mathcal{K}(k)K}{\hat{K}}$.
% \end{enumerate}

% The generalization error results in these cases will be affected by substituting \(\frac{1}{\hat{K}}\) and \(\frac{\mathcal{K}(k)K}{\hat{K}}\) instead of \(\mathcal{K}(k)\) in (\ref{one_round-main}) for method \(1\), and \(2\), respectively.
% Please refer to appendix \ref{partial-ap} for a detailed proof.
% }
\emph{Proof:} The proof of Theorem \ref{one-round_lemma} is in Appendix \ref{one-round-proof-append}. \hfill $\Box$

\begin{remark}
    We note that this theorem and its proof assume that all clients participate in learning. The other scenario is that not all clients participate in the learning procedure. We can consider the following two cases when not all clients participate in the learning procedure.

\noindent \emph{Case I:} Sampling $\hat{K}$ clients with replacement based on distribution $\mathcal{K}$, followed by averaging the local models with equal weights.
    
\noindent \emph{Case II:} Sampling $\hat{K}$ clients without replacement uniformly at random, then performing weighted averaging of local models. Here, the weight of client $k$ is rescaled to $\frac{\mathcal{K}(k)K}{\hat{K}}$.

The generalization error results in these cases are affected by substituting \(\frac{1}{\hat{K}}\) and \(\frac{\mathcal{K}(k)K}{\hat{K}}\) instead of \(\mathcal{K}(k)\) in (\ref{one_round-main}) for cases I and II, respectively. The detailed proof is provided in Appendix \ref{partial-ap}.
\end{remark}

% , \ie  sampling $\hat{K}$ clients with replacement based on distribution $\mathcal{K}$, followed by averaging the local models with equal weights. The other scenario is that not all clients participate in the learning process, \ie Sampling $\hat{K}$ clients without replacement uniformly at random, then performing weighted averaging of local models. Here, the weight of client $k$ is rescaled to $\frac{\mathcal{K}(k)K}{\hat{K}}$. In this scenario, the generalization error results is affected by substituting \(\frac{1}{\hat{K}}\) and \(\frac{\mathcal{K}(k)K}{\hat{K}}\) instead of \(\mathcal{K}(k)\) in (\ref{one_round-main}) for method \(1\), and \(2\), respectively.
% Please refer to appendix \ref{partial-ap} for a detailed proof.

% In the case of partial client participation in FedAvg, we propose two methods:
% \begin{enumerate}
%     \item Sampling $\hat{K}$ clients with replacement based on distribution $\mathcal{K}$, followed by averaging the local models with equal weights.
%     \item Sampling $\hat{K}$ clients without replacement uniformly at random, then performing weighted averaging of local models. Here, the weight of client $k$ is rescaled to $\frac{\mathcal{K}(k)K}{\hat{K}}$.
% \end{enumerate}

% The generalization error results in these cases will be affected by substituting \(\frac{1}{\hat{K}}\) and \(\frac{\mathcal{K}(k)K}{\hat{K}}\) instead of \(\mathcal{K}(k)\) in (\ref{one_round-main}) for method \(1\), and \(2\), respectively.
% Please refer to appendix \ref{partial-ap} for a detailed proof.

\textbf{Discussion.}
Note that there are two terms in the generalization error bound: (i) local generalization of each client that shows more generalizable local models lead to a better generalization of the aggregated model, (ii) non-iidness of each client which deteriorates generalization.
Theorem \ref{one-round_lemma} reveals a factor of $\mathcal{K}(k)^2$ for the first term, which is the sole term in the iid setting. For example, in the uniform case (\(\mathcal{K}(k) = \frac{1}{K}\)), we will observe an improvement with a factor of $\frac{1}{K^2}$ for the iid case. This represents an enhancement compared to the state of the art \citep{Barnes2022ImprovedIG,sun2023understanding,9154277}, which only demonstrates a factor of $\frac{1}{K}$.
As a result, after the averaging process carried out by the central parameter server, the generalization error is reduced by a factor of \(\frac{1}{K^2}\) in iid case.
This is interesting because one would normally expect an improvement of \(\frac{1}{K}\) based on the linear increase in the collective dataset, but this shows an additional improvement of \(\frac{1}{K}\) as well.

On the other hand, we do not see a similar behavior in non-iid case. In other words, the expected generalization error bound does not necessarily decrease with averaging. These results show why FedAvg works well in iid setup but not necessarily in non-iid setup. This observation motivates us to design a new federated learning approach for non-iid setup. The question in this context is what the new federated learning design should be. To answer this question, we analyze $R-$round generalization bound in the next section.

%Note that in a scenario where data is iid, the limit decreases as the number of nodes $K$ increases. As a result, after each averaging process carried out by the central server, the generalization error is reduced by a factor of $K$. This illustrates why, in the context of iid data distribution, the bound experiences a more pronounced reduction through averaging, which contrasts with the situation in a non-iid scenario. This contrast is precisely why the approach proposed in this paper proves advantageous in non-iid scenarios, where averaging doesn't result in a similar scaling of the generalization bound.
% \vspace{-5pt}
\subsection{$R-$Round Generalization Bound} \label{R-round}
% \vspace{-0pt}
% Now we turn our attention to a more complicated scenario; $R$-round FedAvg.
% A similar $R$-round generalization bound analysis is considered in \cite{Barnes2022ImprovedIG}, but without representation learning, which is crucial in our proposed federated learning mechanism in Section \ref{sec:FedALS}. 
In this setup, after $R$ rounds, there is a sequence of weights $\{\hat{{\thetak[]}}_r\}_{r=1}^R$ and the final model is $\hat{{\thetak[]}}_R$. 
We consider that at round $r$, each node constructs its updated model as in (\ref{local_sgd}) by taking $\tau$ gradient steps starting from $\hat{{\thetak[]}}_r$ with respect to $\tau$ random mini-batches $Z_{k,r} = \bigcup \{\mathcal{B}_{k,r,t}\}_{t=0}^{\tau-1}$ drawn from the local dataset $\Sk$.
For this type of iterative algorithm, we consider the following averaged empirical risk
{\fontsize{9.5}{8}\selectfont
\begin{align}
    \frac{1}{R} \sum_{r=1}^R \EX_{k \sim \mathcal{K}} \bigg[\frac{1}{|Z_{k,r}|}\sum_{i \in Z_{k,r}} l(\Mk[\hat{\thetak[]}_r],\z{k}{i})\bigg].
\end{align}
}
The corresponding generalization error, $\Delta_{FedAvg}(\Sk[])$, is
{\fontsize{9.2}{8}\selectfont
\begin{align}
    \hspace{-2pt} \frac{1}{R} \hspace{-2pt} \sum_{r=1}^R \hspace{-1pt}\EX_{k \sim \mathcal{K}} \hspace{-3pt}\bigg[\hspace{-2pt}\EX_{\Zonly \sim \mathcal{D}_k} \hspace{-3pt}l(\Mk[\hat{\thetak[]}_r],\Zonly) \label{fedAvg_g} \hspace{-2pt} - \hspace{-2pt} \frac{1}{|Z_{k,r}|}\hspace{-3pt}\sum_{i \in Z_{k,r}} \hspace{-4pt}l(\Mk[\hat{\thetak[]}_r],\z{k}{i})\hspace{-1pt}\bigg].
\end{align}}
Note that the expression in (\ref{fedAvg_g}) differs slightly from the end-to-end generalization error that would be obtained by considering the final model $\Mk[\hat{\thetak[]}_R]$ and the entire dataset $\Sk[]$. More specifically, (\ref{fedAvg_g}) is an average of the generalization errors measured at each round, similar to \cite{Barnes2022ImprovedIG}).
We anticipate that the generalization error diminishes with the increasing number of data samples, so this generalization error definition yields a more cautious upper limit and serves as a sensible measure. The next theorem characterizes the expected generalization error bounds for $R-$Round FedAvg in iid and non-iid settings. 

%As our anticipation is for the generalization error to diminish with the utilization of more samples, this measurement should yield a more cautious upper limit and serve as a sensible substitute for examination.

\begin{theorem}\label{mani_th}
Let $l(\Mk[{\Thetak[]}],\Zonly)$ be $\mu$-strongly convex and $L$-smooth in $\Mk[{\Thetak[]}]$.
Local models at round $r$ are calculated by doing $\tau$ local gradient descent steps (\ref{local_sgd}), and the local gradient variance is bounded by $\sigma^2$, \ie $\EX_{\Zonly \sim \mathcal{D}_k} \norm{\nabla l(\Mk[{\Thetak[]}],\Zonly) - \EX_{\Zonly \sim \mathcal{D}_k}\nabla l(\Mk[{\Thetak[]}],\Zonly)}\leq \sigma^2$.
The aggregated model at round $r$, $\Mk[\hat{{\Thetak[]}}_r]$, is obtained by performing FedAvg, and the data points used in round $r$ (\ie $Z_{k,r}$) are sampled without replacement.
Then the average generalization error, $\EX_{\Sk[]} \Delta_{FedAvg}(\Sk[])$, is upper bounded by 
{\fontsize{9.5}{8}\selectfont
\begin{align}
    \frac{1}{R} \sum_{r=1}^R &\EX_{k \sim \mathcal{K}}\bigg[\frac{2L\mathcal{K}(k)^2}{\mu} A  \label{main}+ \sqrt{\frac{8L}{\mu}}\mathcal{K}(k)(AB)^{\frac{1}{2}}\bigg],
\end{align}
}
where,
{\fontsize{9.5}{8}\selectfont
\begin{align}
    &A = \Tilde{O}\bigg(\sqrt{\frac{\mathcal{C}(\Mk[{\Thetak[]}])}{|Z_{k,r}|}} + \frac{\sigma^2}{\mu \tau} \bigg),\nonumber\\
    &B= \Tilde{O}\bigg( \EX_{\{Z_{k,r}\}_{k=1}^K} \delta_{k,\mathcal{A}}(\{Z_{k,r}\}_{k=1}^K) + \frac{\sigma^2}{\mu \tau} \bigg). \nonumber
\end{align}
}
$\Tilde{O}$ hides constants and poly-logarithmic factors,  and $\mathcal{C}(\Mk[{\Thetak[]}])$  shows the complexity of the model class of $\Mk[{\Thetak[]}]$.
\end{theorem}
\emph{Proof:} The proof of Theorem \ref{mani_th} is in Appendix \ref{ap-b}. \hfill $\Box$

The generalization error bound in  (\ref{main}) depends on the following parameters: (i) number of rounds; $R$, (ii) number of samples used in every round; $|Z_{k,r}|$, (iii) the complexity of the model class; $\mathcal{C}(\Mk[{\Thetak[]}])$, non-iidness; $\delta_{k,\mathcal{A}}(\{Z_{k,r}\}_{k=1}^K)$, number of local steps in each round; $\tau$.
We note that (\ref{main}) also depends on $K$ (more specifically $\mathcal{K}$), but this dependence is similar to the discussion we had for one-round generalization, so we skip it here. 

\begin{algorithm}[tb]
\caption{FedALS}
\label{alg:algorithm1}
\textbf{Input}: Initial model $\{\thetak[k,0,0]=[\phik[k,0,0],\hk[k,0,0]]\}_{k=1}^K$, learning rate $\eta$, number of local steps for the head $\tau$, adaptation coefficient $\alpha$. 
\begin{algorithmic}[1] %[1] enables line numbers
\FOR{Round $r$ in $0,...,R-1$}
\FOR{Node $k$ in $1,...,K$ \textbf{in parallel}}
\FOR{Local step $t$ in $0,...,\tau -1$}
\STATE Sample the batch $\mathcal{B}_{k,r,t}$ from $\mathcal{D}_k$.
\STATE $\thetak[{k,r,t+1}]$ $=$ $\thetak[{k,r,t}]$ $-$ $\frac{\eta}{|\mathcal{B}_{k,r,t}|}$  $\sum_{i \in \mathcal{B}_{k,r,t}}$ $\nabla l(\mk[{\thetak[{k,r,t}]}],\z{k}{i})$
\IF {$\mod(r\tau + t , \tau)  = 0 $}
\STATE Send $\hk[{k,r,t}]$ to the server.
\STATE Receive the aggregated head: $\hk[{k,r,t}] \leftarrow \frac{1}{K} \sum_{k=1}^K\hk[{k,r,t}]$
\ENDIF
\IF {$\mod(r\tau + t ,  \alpha\tau)  = 0 $}
\STATE Send $\phik[{k,r,t}]$ to the server.
\STATE Receive the aggregated Representation extractor: $\phik[{k,r,t}] \leftarrow \frac{1}{K} \sum_{k=1}^K\phik[{k,r,t}]$
\ENDIF
\ENDFOR
\STATE $\thetak[{k,r+1,0}] = \thetak[{k,r,\tau}]$
\ENDFOR
\ENDFOR
\STATE \textbf{return} $\hat{{\thetak[]}}_R = \frac{1}{K} \sum_{k=1}^K\thetak[{k,R,0}]$
\end{algorithmic}
\end{algorithm}

% \vspace{-10pt}
% \section{Representation Learning Interpretation of R-Round Generalization Bound}
% \vspace{-3pt}

% \vspace{-10pt}
\section{\label{sec:FedALS} FedALS: Federated Learning with Adaptive Local Steps}

The number of samples and local steps in each round are crucial for minimizing the generalization error bound, particularly in non-iid scenarios, as described in (\ref{main}), where the error bound is looser compared to the iid setup. Our key insight in this paper is that reducing the aggregation frequency, which increases both \(\tau\) and \(|Z_{k,r}|\), can lead to a smaller generalization error bound according to (\ref{main}).

Increasing \(\tau\) also increases the non-iidness among nodes, defined as:
\[
\delta_{k,\mathcal{A}}(\{Z_{k,r}\}_{k=1}^K) =  R_{Z_{k,r}}(\mk[{\hk[]}] \circ \mk[{\phik[]}]) - R_{Z_{k,r}}(\mk[{\hk[k]}] \circ \mk[{\phik[k]}]),
\]
where \(\mk[{\hk[]}] \circ \mk[{\phik[]}]\) represents the aggregated model \(\mathcal{A}(\{Z_{k,r}\}_{k=1}^K)\), and \(\mk[{\hk[k]}] \circ \mk[{\phik[k]}]\) represents the local model \(\mathcal{A}_k(Z_{k,r})\). This happens because the model moves closer to the local optimum of node \(k\)'s loss, thereby reducing \(R_{Z_{k,r}}(\mk[{\hk[k]}] \circ \mk[{\phik[k]}])\).

In the context of representation learning, as discussed in the Introduction, the representation extractors \(\mk[{\phik[]}]\) and \(\mk[{\phik[k]}]\) are much more similar compared to the task-specific heads \(\mk[{\hk[]}]\) and \(\mk[{\hk[k]}]\). Therefore, increasing the number of local steps for the representation extractor does not significantly increase non-iidness, unlike increasing local steps for the head.

The main idea behind FedALS is to improve generalization of the model by increasing the number of local steps while maintaining a low level of non-iidness. This can be achieved if $\tau_{\Mk[{\phik[]}]}$ is set larger than $\tau_{\Mk[{\hk[]}]}$, where $\tau_{\Mk[{\phik[]}]}$ and $\tau_{\Mk[{\hk[]}]}$ are the corresponding number of local iterations for $\Mk[{\phik[]}]$ and $\mk[{\hk[]}]$, respectively.
FedALS in Algorithm \ref{alg:algorithm1} divides the model into two parts: (i) the representation extractor, denoted as $\Mk[{\phik[]}]$, and (ii) the head, denoted as $\Mk[{\hk[]}]$. 
% The aggregation interval for $\Mk[{\phik[]}]$ and $\Mk[{\hk[]}]$ are $\tau_{\Mk[{\phik[]}]}$ and $\tau_{\Mk[{\hk[]}]}$, respectively.
Additionally, we introduce the parameter $\alpha = \frac{\tau_{\Mk[{\phik[]}]}}{\tau_{\Mk[{\hk[]}]}}$ as an adaptation coefficient, which can be regarded as a hyperparameter for estimating the true ratio.

% \vspace{-10pt}
\section{Experimental Results}
\vspace{-2pt}

In this section, we assess the performance of FedALS using ResNet-$20$ as a deep neural network architecture and OPT-$125$M as a large language model. We treat the convolutional layers of ResNet-$20$ as the representation extractor and the final dense layers as the model head. For OPT-$125$M, we consider the first $10$ layers of the model as the representation extractor.
We used the datasets CIFAR-$10$, CIFAR-$100$, SVHN, and MNIST for image classification and the Multi-Genre Natural Language Inference (MultiNLI) corpus for the LLM.
The experimentation was conducted on a network consisting of five nodes alongside a central server. For image classification, we utilized a batch size of $64$ per node. SGD with momentum was employed as the optimizer, with the momentum set to $0.9$, and the weight decay to $10^{-4}$.
For the LLM fine-tuning, we employed a batch size of 16 sentences from the corpus, and the optimizer used was AdamW.
In all the experiments, to perform a grid search for the learning rate, we conducted each experiment by multiplying and dividing the learning rate by powers of two, stopping each experiment after reaching a local optimum learning rate.
% The experiments are conducted on Ubuntu $20.04$  using $36$ Intel Core i9-10980XE processors and $5$ GeForce RTX 2080 graphics cards. 
We repeat each experiment $20$ times and present the error bars associated with the randomness of the optimization. In every figure, we include the average and standard deviation error bars. Detailed experimental setup is provided in Appendix \ref{detaild-exp} of the supplementary materials.

% We selected a base value of $\tau$ to represent the length of local steps in FedAvg. In order to determine the values of $\alpha_l$ for layers ranging from $1$ to $L$, we introduced an additional parameter denoted as $\alpha$. This constant is the ratio of the length of local steps for the initial layer to the last layer. The length of local steps for layers from the 2nd to the $(l-1)$th layer follows a linear decrease pattern, determined by the values assigned to the first and last layers.

\vspace{-5pt}
\subsection{FedALS in non-iid Setting}
\vspace{-2pt}
In this section, we allocate the dataset to nodes using a non-iid approach.
For image classification, we initially sorted the data based on their labels and subsequently divided it among nodes following this sorted sequence.
In MultiNLI, we sorted the sentences based on their genre.

\begin{figure}[tb]
     \centering
     \begin{subfigure}[b]{0.49\textwidth}
         \centering
        \includegraphics[width=\textwidth]{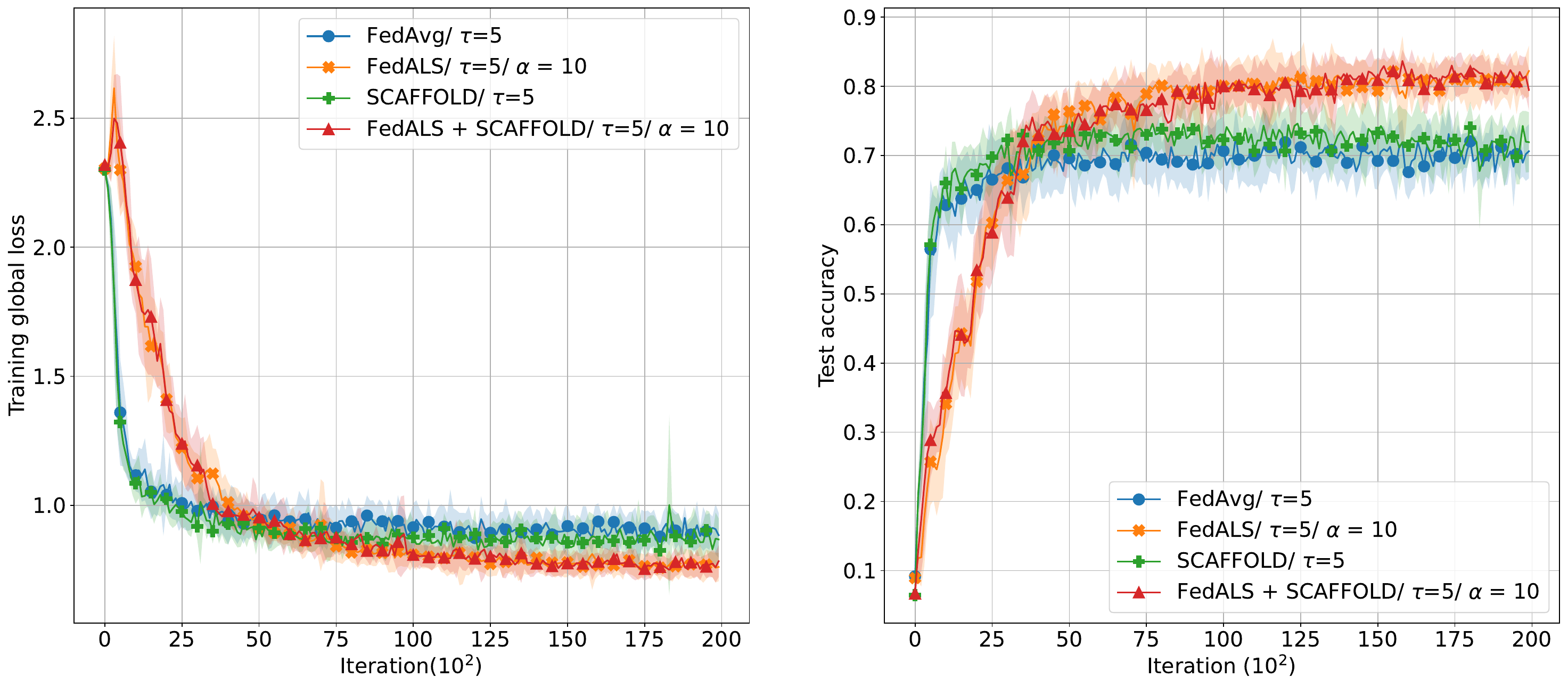}
         \caption{non-iid}
         \label{SVHN_noniid}
     \end{subfigure}
     \hfill
     \begin{subfigure}[b]{0.49\textwidth}
         \centering
        \includegraphics[width=\textwidth]{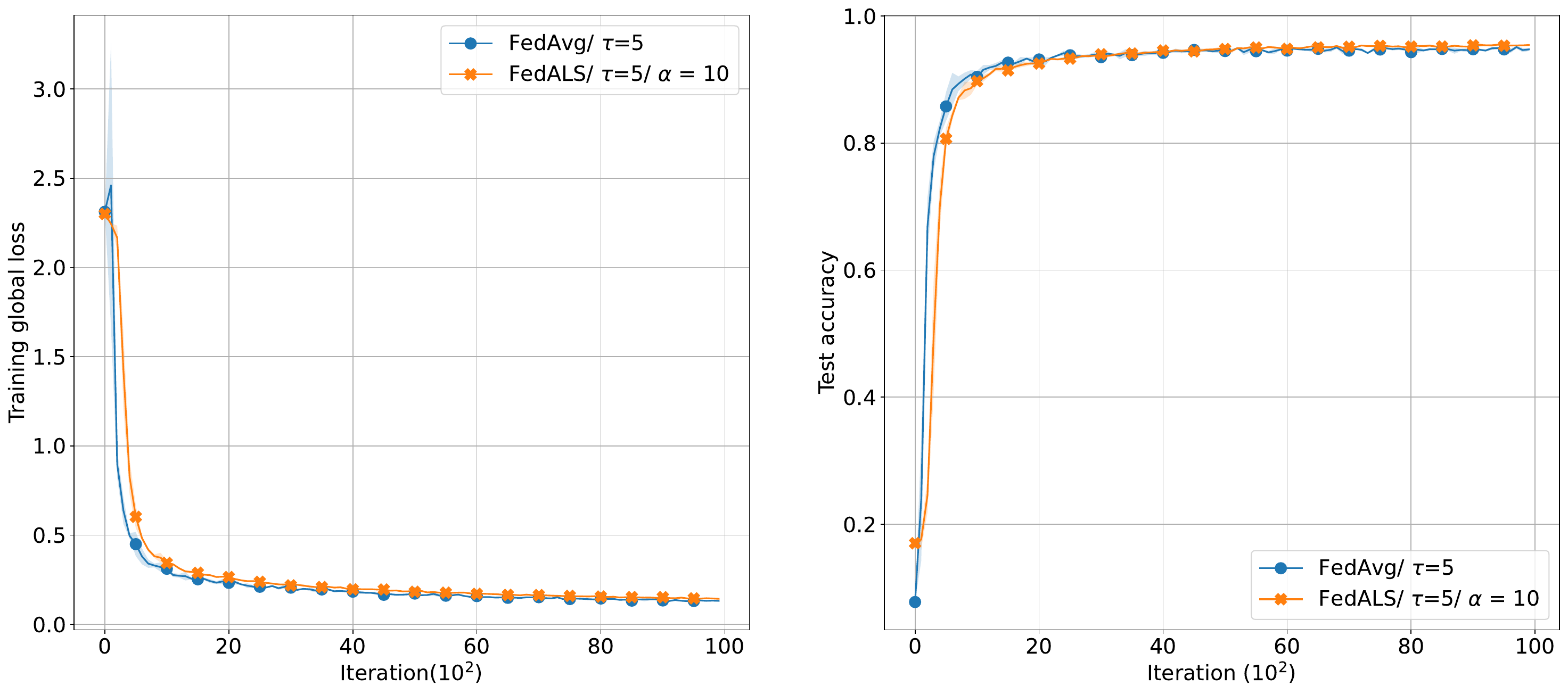}
         \caption{iid}
         \label{SVHM_iid}
     \end{subfigure}
        % \vspace{-20pt}
        \caption{Training  ResNet-$20$ on SVHN.}
        \label{SVHN}
        \vspace{-10pt}
\end{figure}

\begin{figure}[tb]
     \centering
     \begin{subfigure}[b]{0.49\textwidth}
         \centering     \includegraphics[width=\textwidth]{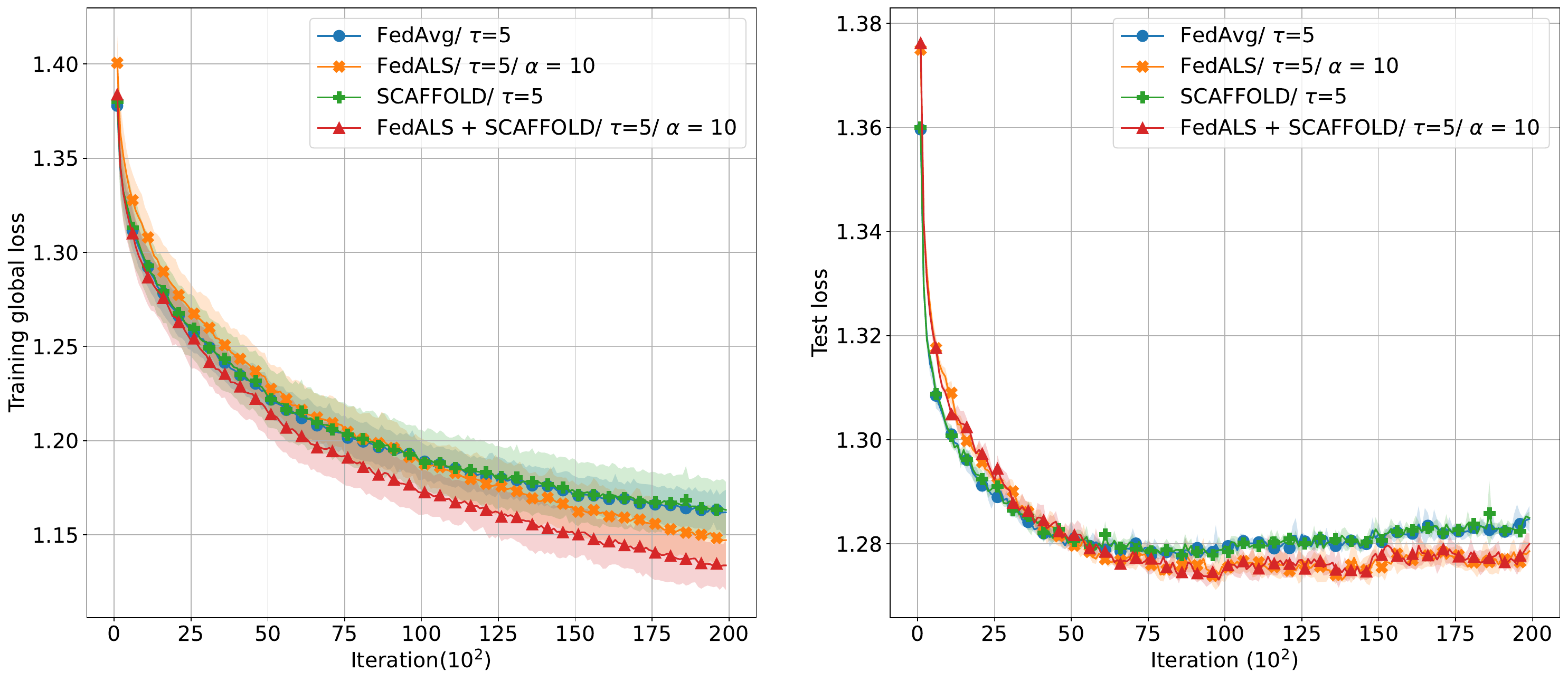}
         \caption{non-iid}
         \label{llm_noniid}
     \end{subfigure}
     \hfill
     \begin{subfigure}[b]{0.49\textwidth}
         \centering
         \includegraphics[width=\textwidth]{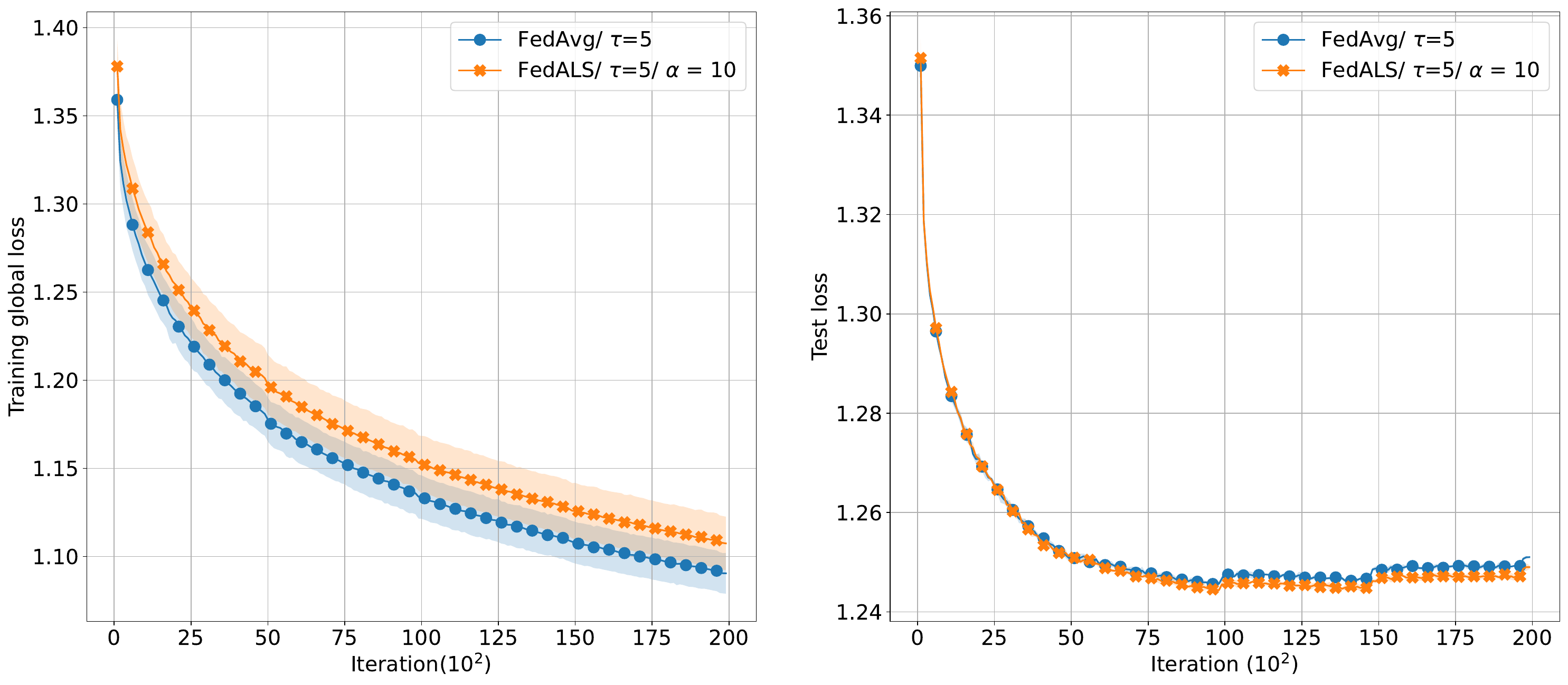}
         \caption{iid}
         \label{llm_iid}
     \end{subfigure}
        % \vspace{-20pt}
        \caption{Fine-tuning OPT-$125$M on MultiNLI.}
        \label{llm}
        \vspace{-10pt}
\end{figure}

% \begin{figure}[tb]
%      \centering
%      \begin{subfigure}[b]{0.49\textwidth}
%          \centering     \includegraphics[width=\textwidth]{experiments/ICML_paper_alpha_iidFalse_cifar10.pdf}
%          \caption{non-iid}
%          \label{cifar10_noniid}
%      \end{subfigure}
%      \hfill
%      \begin{subfigure}[b]{0.49\textwidth}
%          \centering
%         \includegraphics[width=\textwidth]{experiments/ICLR_paper_iidTrue_cifar10_scaffold.pdf}
%          \caption{iid}
%          \label{cifar10_iid}
%      \end{subfigure}
%         \vspace{-9pt}
%         \caption{Training  ResNet-$20$ on CIFAR-$10$.}
%         \label{cifar10}
%         \vspace{-15pt}
% \end{figure}

% \begin{figure}[t]
%      \centering
%      \begin{subfigure}[b]{0.49\textwidth}
%          \centering
%         \includegraphics[width=\textwidth]{iclr2024/experiments/ICLR_paper_iidFalse_cifar100_scaffold.pdf}
%          \caption{non-iid}
%          \label{cifar100_noniid}
%      \end{subfigure}
%      \hfill
%      \begin{subfigure}[b]{0.49\textwidth}
%          \centering
%         \includegraphics[width=\textwidth]{iclr2024/experiments/ICLR_paper_iidTrue_cifar100_scaffold.pdf}
%          \caption{iid}
%          \label{cifar100_iid}
%      \end{subfigure}
%         \caption{Results for CIFAR-$100$.}
%         \label{cifar100}
%         \vspace{-5pt}
% \end{figure}

In this scenario, we can observe in Fig.  \ref{SVHN_noniid}, \ref{llm_noniid} the anticipated performance improvement through the incorporation of different local steps across the model. By utilizing parameters $\tau = 5$ and $\alpha = 10$ in FedALS, it becomes apparent that aggregation and communication costs are reduced compared to FedAvg with the same $\tau$ value of 5.
This implies that the initial layers perform aggregation at every $50$ iterations.
This reduction in the number of communications is accompanied by enhanced model generalization stemming from the larger number of local steps in the initial layers, which contributes to an overall performance enhancement.
Thus, our approach in FedALS is beneficial for both communication efficiency and enhancing model generalization performance simultaneously.
% Note that we have conducted the experiment using FedAvg with a $\tau$ value of 500 as well. It is clear that in this scenario, the consensus distance becomes significantly large, leading to a negative impact on the performance.

% We note that the enhancement in performance is contingent upon the degree of non-iidness exhibited by the local datasets and the impact of consensus distance on the convergence. As a consequence, the selection of an appropriate value for $\alpha$ holds significance, intertwined with the model and system parameters. Thus, we consider  $\alpha$ as a hyper-parameter, so its value should be adjusted similar to other hyper-parameters like learning rate. 
%This is similar to adjusting a parameter, like changing the learning rate.
%As we will observe in the upcoming sub-section, when the data is iid, the previously noted performance enhancement that is achieced alongsied the communication efficiency, which is inherited in FedALS, no longer manifest.

\vspace{-5pt}
\subsection{FedALS in iid Setting}
\vspace{-2pt}
The results for the iid setting are presented in Fig. \ref{SVHM_iid}, \ref{llm_iid}. In order to obtain these results, the data is shuffled, and then evenly divided among nodes. We note that in this situation, the performance improvement of FedALS is negligible. This is expected since there is a factor of $\frac{1}{K^2}$ in the generalization in this case, ensuring that we will have nearly the same population risk as the empirical risk. 

% Therefore, the deciding factor here is the optimization of the empirical risk, which is improved with a smaller $\tau$ as discussed in Remark \ref{note}. %\textcolor{red}{We expected this result as the generalization error bound (\ref{iid}) decreases after each aggregation on the order of $K$}. \textcolor{green}{HS: It is $K^2$ now, right?} 
% Thus, the improvement of the generalization error using the FedALS approach is negligible in this setup. 

%Due to the diminished generalization bound after each averaging (\ref{iid}) with $K$, the advantage of a more generalizable model is not as evident. % in contrast to the increased consensus distance.

%This result is intuitive and expected. In the iid setting, after each averaging step, the combined model reflects the anticipated generalization error across the entire set of $nK$ samples, rather than just $n$ samples. Consequently, each averaging step contributes to minimizing the generalization error. As a result, we cannot anticipate a comparable enhancement in terms of generalization improvement as witnessed in the non-iid scenario.

\vspace{-5pt}
\subsection{Compared to and Complementing SCAFFOLD}
\vspace{-2pt}
\citet{Karimireddy2019SCAFFOLDSC} introduced an innovative technique called SCAFFOLD, which employs some control variables for variance reduction to address the issue of ``client-drift'' in local updates. This drift happens when data is heterogeneous (non-iid), causing individual nodes/clients to converge towards their local optima rather than the global optima. While this approach is a significant theoretical advancement in achieving independence from loss function disparities among nodes, it hinges on the assumption of smoothness in the loss functions, which might not hold true for practical deep learning problems in the real world.
Additionally, since SCAFFOLD requires the transmission of control variables to the central server, which is of the same size as the models themselves, it results in approximately twice the communication overhead when compared to FedAvg.

Let us consider Fig. \ref{SVHN_noniid}, \ref{llm_noniid} to notice that in real-world deep learning situations, FedALS enhances performance significantly, while SCAFFOLD exhibits slight improvements in specific scenarios. Moreover, we integrated FedALS and SCAFFOLD to concurrently leverage both approaches. The results of the test accuracy in different datasets are summarized in Table \ref{Tab}.

{

\subsection{Complementing SCAFFOLD}
\vspace{-2pt}
\citet{Karimireddy2019SCAFFOLDSC} introduced an innovative technique called SCAFFOLD, which employs some control variables for variance reduction to address the issue of ``client-drift'' in local updates. This drift happens when data is heterogeneous (non-iid), causing individual nodes/clients to converge towards their local optima rather than the global optima.
Methods like SCAFFOLD are complementary to FedALS. They can be applied concurrently to gain the dual benefits of improved communication efficiency from FedALS while further combating the effects of non-iid data.
SCAFFOLD incurs approximately double the communication overhead of FedAvg because it requires sending control variables that are equal in size to the model parameters. However, this overhead can be mitigated by applying FedALS, which improves SCAFFOLD's communication efficiency.

Let us consider Fig. \ref{SVHN_noniid}, \ref{llm_noniid} to notice that in real-world deep learning situations, FedALS enhances performance significantly, while SCAFFOLD exhibits slight improvements in specific scenarios. Moreover, we integrated FedALS and SCAFFOLD to concurrently leverage both approaches. The results of the test accuracy in different datasets are summarized in Table \ref{Tab}.
}

\begin{table*}[tb]
\centering
\renewcommand{\arraystretch}{1.05} % Adjust row spacing
\setlength{\tabcolsep}{2.5pt} % Reduce column spacing
\caption{Test Performance for 5 nodes FL; Accuracy after training ResNet-$20$ and test loss after fine-tuning OPT-$125$M in iid and non-iid settings with $\tau=5$ and $\alpha = 10$.}
\label{Tab}
\scriptsize % Reduce font size
\resizebox{\textwidth}{!}{ % Auto-adjust width to fit page
\begin{tabular}{lcccccc}
    \toprule
    \multirow{2}{*}{Model/Dataset} & \multicolumn{2}{c}{FedAvg} & \multicolumn{2}{c}{FedALS} & SCAFFOLD & FedALS + SCAFFOLD \\
    \cmidrule(lr){2-3} \cmidrule(lr){4-5} \cmidrule(lr){6-6} \cmidrule(lr){7-7}
    {} & IID  & Non-IID  & IID & Non-IID &  Non-IID &  Non-IID \\
    \midrule
    ResNet-$20$/SVHN  &  $0.948 \pm 0.002$ & $\mathbf{0.701 \pm 0.033}$  & $0.954 \pm 0.002$ & $\mathbf{0.812 \pm 0.021}$ &  $0.718 \pm 0.040$ & $0.810 \pm 0.027$\\
    ResNet-$20$/CIFAR-$10$  &$0.876 \pm 0.009$ & $\mathbf{0.465 \pm 0.007}$  & $0.887 \pm 0.002$ & $\mathbf{0.522 \pm 0.037}$&  $0.454 \pm 0.071$ & $0.512 \pm  0.010$\\
    ResNet-$20$/CIFAR-$100$  &$0.600 \pm 0.007$ & $\mathbf{0.418 \pm 0.014}$  & $0.612 \pm 0.005$ & $\mathbf{0.486 \pm 0.022}$&  $0.412 \pm 0.018$ & $0.482 \pm 0.013$\\
    ResNet-$20$/MNIST  &$0.991 \pm 0.000$ & $\mathbf{0.797 \pm 0.064}$  & $0.991 \pm 0.000$ & $\mathbf{0.821 \pm 0.036}$&  $0.812 \pm 0.055$ & $0.783 \pm 0.118$\\
    OPT-$125$M/MultiNLI  &$1.250\pm 0.000$ & $\mathbf{1.284\pm 0.002}$  & $1.248 \pm 0.001$ & $\mathbf{1.277\pm 0.002}$&  $1.284\pm 0.001$ & $1.278\pm 0.002$\\
    \bottomrule
\end{tabular}
}
\end{table*}

\vspace{-3pt}
\subsection{The Role of \boldmath$\alpha$ and Communication Overhead}
\vspace{-2pt}

% \begin{wrapfigure}{r}{0.40\textwidth}
% % \begin{figure}[]
% \centering
% \includegraphics[width=0.4\textwidth]{iclr2024/experiments/ICLR_paper_alpha_iidFalse_SVHN_scaffold.pdf} % Reduce the figure size so that it is slightly narrower than the column.
% \vspace{-10pt}
% \caption{The role of $\alpha$ in training ResNet-$20$ for SVHN with $\tau=5$.}
% \vspace{-30pt}
% \label{alpha}
% % \end{figure}
% \end{wrapfigure}

As shown in Table \ref{alpha2}, it becomes evident that when we increase $\alpha$ from $1$ (FedAvg), we initially witness an enhancement in accuracy owing to improved generalization. However, beyond a certain threshold ($\alpha =10$), further increment in $\alpha$ ceases to contribute to performance improvement. This is due to the adverse impact of a high number of local steps on the non-iidness. The trade-off we discussed in the earlier sections is evident in this context. We have also demonstrated the impact of FedALS on the communication overhead in this table.

\begin{table}[t] 
\bgroup
\def\arraystretch{1.1}
\caption{The accuracy and communication overhead per client after training ResNet-$20$ in non-iid setting with $\tau=5$ and variable $\alpha$.}
% \vspace{-17pt}
\label{alpha2}
%\vskip 0.15in
% \vspace{-10pt}
\begin{center}
{
\begin{sc}
\begin{tabular}{@{\hskip 1pt}c@{\hskip -1pt}c@{\hskip 6pt}c@{\hskip -7pt}c@{\hskip 1pt}}
\toprule
\multirow{2}{*}[-.4em]{Value of \boldmath$\alpha$} & \multicolumn{2}{c}{Dataset } &\multirow{1}{*}[-.7em]{\# of communicated}\\
\addlinespace[1pt] \cmidrule(lr){2-3} \\
\addlinespace[-10pt] 
${}$ & SVHN & CIFAR-$10$ &parameters  \\
\midrule
$1$ & $0.7010 \pm 0.0330$ & $0.4651 \pm 0.0071$  & $2.344B$\\
$5$ & $0.8107 \pm 0.0278$& $0.5201 \pm 0.0302$  & $0.473B$\\
$10$ & $\mathbf{0.8117 \pm  0.0214}$ & $\mathbf{0.5224 \pm 0.0365}$  & $0.239B$ \\
$25$ & $0.7201 \pm 0.0549$ & $ 0.3814 \pm 0.0641$  & $0.099B$\\
$50$ & $0.6377 \pm 0.0520$ & $0.2853 \pm 0.0641$  & $0.052B$\\
$100$ & $0.5837 \pm 0.0715$ & $0.2817 \pm 0.032$  & $0.029B$\\
\bottomrule
% \vspace{10pt}
\end{tabular} 
\vspace{-10pt}
\end{sc}
}
\end{center}
% \vskip -0.4in
\egroup
\end{table}

\vspace{-5pt}
\subsection{Different Combinations of $\phik[],\hk[]$}
\vspace{-2pt}
In Table \ref{alpha3}, we have presented the results of our experiments, illustrating how different combinations of $\phik[]$ and $\hk[]$ influence the model performance in FedALS. The parameter $L$ indicates the number of layers in the model considered as the representation extractor ($\phik[]$), while the remaining layers are considered as $\hk[]$. We observe that for ResNet-$20$, choosing $\phik[]$ to be the first $16$ layers and performing less aggregation for them seems to be the most effective option.

\begin{table}[t] 
\bgroup
\def\arraystretch{1.1}
\caption{Different Combinations of $\phik[],\hk[]$ for training ResNet-$20$ in non-iid setting with $\tau=5$, $\alpha=10$.}
% \vspace{-16pt}
\label{alpha3}
%\vskip 0.15in
% \vspace{-5pt}
\begin{center}
{
\begin{sc}
\begin{tabular}{@{\hskip -2pt}c@{\hskip -4pt}c@{\hskip 6pt}c@{\hskip 6pt}c@{\hskip 1pt}}
\toprule
\multirow{2}{*}[-.4em]{Value of \boldmath$L$} & \multicolumn{3}{c}{Dataset } \\
\addlinespace[1pt] \cmidrule(lr){2-4} \\
\addlinespace[-10pt] 
${}$ & SVHN & CIFAR-$10$ & CIFAR-$100$  \\
\midrule
$20$ & $0.6991 \pm 0.0160$ & $0.4383 \pm 0.0423$  & $0.4781 \pm 0.0123$\\
$16$ & $\mathbf{0.7112 \pm 0.0471}$& $\mathbf{0.4687 \pm 0.0111}$  & $\mathbf{0.4782 \pm 0.0087}$\\
$12$ & $0.6760 \pm  0.0474$ & $0.4125 \pm 0.0283$  & $0.4249 \pm 0.0143$ \\
$8$ & $0.6381 \pm 0.0428$ & $ 0.3779 \pm 0.03451$  & $0.4085 \pm 0.0094$\\
$4$ & $0.6339 \pm 0.0446$ & $0.3730 \pm 0.0310$  & $0.4183 \pm 0.0108$\\
$1$ & $0.6058 \pm 0.0197$ & $0.4013 \pm 0.0308$  & $0.3880 \pm 0.0305$\\
\bottomrule
\vspace{-30pt}
\end{tabular} 
\end{sc}
}
\end{center}
% \vskip -0.3in
\egroup
\end{table}
% \vspace{-10pt}
% \section{Acknowledgments}
% \vspace{-5pt}
% This work was supported in part by ARL under Grant W911NF-2120272, and in part by NSF under Grant CCF-1942878, Grant CNS-2148182, and Grant CNS-2112471.
\vspace{-5pt}
\section{Conclusion}
\vspace{-2pt}
In this paper, we characterized the generalization error bound for one- and R-round federated learning, showing that the one-round bound is tighter than the current state of the art. Our analysis, combined with a representation learning perspective, revealed that less frequent aggregations, resulting in more local updates for the initial layers, lead to more generalizable models, especially in non-iid scenarios. This insight inspired the development of the FedALS algorithm, which increases local steps for the initial layers while performing more averaging for the final layers. Experimental results demonstrated FedALS's effectiveness in heterogeneous setups.

\subsubsection*{Acknowledgments}
This work was supported in parts by the Army Research Lab (W911NF2420172), and the National Science Foundation (CCF-1942878, CNS-2148182, CNS-2112471).

\bibliography{main}
\bibliographystyle{tmlr}

\appendix
\section{Appendix}

\section{Proof of Theorem \ref{one-round_lemma}}\label{one-round-proof-append}

We first state and prove the following lemma that will be used in the proof of Theorem \ref{one-round_lemma}.

\begin{lemma}[Leave-one-out] \label{LOO}
[Expansion of Theorem 1 in \cite{Barnes2022ImprovedIG}]

Let $\Sk' = (\Z{k}{1}',...,\Z{k}{n_k}')$, where $\Z{k}{i}'$ is sampled from $\mathcal{D}_k$. 
% Let also define $\Sk[]' = (\Sk[1]',...,\Sk[K]')$.
% Denote $\Sk^{(i)} = (\Z{k}{1},...,\Z{k}{i}',...,\Z{k}{n_k})$  to represent the alteration of $\Sk$ wherein the $i$-th data sample is replaced with $\Z{k}{i}'$. 
Denote $\Sk[]^{(k)} = (\Sk[1],...,\Sk',...,\Sk[K])$.
Then
\begin{align}
    \EX_{\{\Sk[k] \sim \mathcal{D}_k^{n_k}\}_{k=1}^K} \Delta _{\mathcal{A}}(\Sk[]) = \EX_{k \sim \mathcal{K},\{\Sk[k], \Sk[k]' \sim \mathcal{D}_k^{n_k}\}_{k=1}^K}\bigg[\frac{1}{n_k} \sum_{i=1}^{n_k} \bigg(l(\mathcal{A}(\Sk[]),\Z{k}{i}') -  l(\mathcal{A}(\Sk[]^{(k)}),\Z{k}{i}')\bigg)\bigg].
\end{align}
\end{lemma}

\begin{proof}
We have
\begin{align}\label{part1}
\EX_{\{\Sk[k] \sim \mathcal{D}_k^{n_k}\}_{k=1}^K} R(\mathcal{A}(\Sk[])) =\EX_{k \sim \mathcal{K}, \{\Sk[k], \Sk[k]' \sim \mathcal{D}_k^{n_k}\}_{k=1}^K} l(\mathcal{A}(\Sk[]),\Z{k}{i}').
\end{align}
Also, observe that 
\begin{align}
    \EX_{\{\Sk[k] \sim \mathcal{D}_k^{n_k}\}_{k=1}^K} R_{\Sk[]}(\mathcal{A}(\Sk[])) &=\EX_{k \sim \mathcal{K},\{\Sk[k] \sim \mathcal{D}_k^{n_k}\}_{k=1}^K} \bigg[\frac{1}{n_k} \sum_{i=1}^{n_k} l(\mathcal{A}(\Sk[]),\Z{k}{i})\bigg]\\
    &=\EX_{k \sim \mathcal{K},\{\Sk[k], \Sk[k]' \sim \mathcal{D}_k^{n_k}\}_{k=1}^K} \bigg[\frac{1}{n_k} \sum_{i=1}^{n_k} l(\mathcal{A}(\Sk[]^{(k)}),\Z{k}{i}')\bigg].\label{part2}
\end{align}

Putting \ref{part1}, and \ref{part2} together, and by the definition of the expected generalization error, we get
\begin{align}
    \EX_{\{\Sk[k] \sim \mathcal{D}_k^{n_k}\}_{k=1}^K} \Delta _{\mathcal{A}}(\Sk[]) = \EX_{k \sim \mathcal{K},\{\Sk[k], \Sk[k]' \sim \mathcal{D}_k^{n_k}\}_{k=1}^K}\bigg[\frac{1}{n_k} \sum_{i=1}^{n_k} \bigg(l(\mathcal{A}(\Sk[]),\Z{k}{i}') -  l(\mathcal{A}(\Sk[]^{(k)}),\Z{k}{i}')\bigg)\bigg].
\end{align}
\end{proof}

In the following lemma, we establish a fundamental generalization bound for a single round of ERM and FedAvg. (Theorem \ref{one-round_lemma}). 

\begin{theorem}
\label{one-round_lemma_ap}
    Let $l(\Mk[{\Thetak[]}],\Zonly)$ be $\mu$-strongly convex and $L$-smooth in $\Mk[{\Thetak[]}]$, $\Mk[{\Thetak[]}_k]=\mathcal{A}_k(\Sk)$ represents the model obtained from Empirical Risk Minimization (ERM) algorithm on local dataset $\Sk$, \ie $\Mk[{\Thetak[]}_k]= \arg \min_{\Mk[]} \sum_{i=1}^{n_k} l(\Mk[],\z{k}{i})$, and  $\Mk[\hat{{\Thetak[]}}]=\mathcal{A}(\Sk[])$ is the model after one round of FedAvg ($\hat{{\Thetak[]}} = \EX_{k \sim \mathcal{K}} \Thetak$).
    Then, the expected generalization error, $\EX_{\{\Sk[k] \sim \mathcal{D}_k^{n_k}\}_{k=1}^K}\Delta_{\mathcal{A}}(\Sk[])$, is bounded by
    \begin{align}
        \EX_{k \sim \mathcal{K}}\bigg[\frac{L\mathcal{K}(k)^2}{\mu}\underbrace{\EX_{\{\Sk[k] \sim \mathcal{D}_k^{n_k}\}} \Delta_{\mathcal{A}_k}(\Sk)}_\text{Expected local generalization} + 2\sqrt{\frac{L}{\mu}}\mathcal{K}(k)\underbrace{\sqrt{\EX_{\{\Sk[k] \sim \mathcal{D}_k^{n_k}\}_{k=1}^K} \delta_{k,\mathcal{A}}(\Sk[])}}_\text{Root of expected non-iidness} \underbrace{\sqrt{\EX_{\{\Sk[k] \sim \mathcal{D}_k^{n_k}\}}\Delta_{\mathcal{A}_k}(\Sk)}}_\text{Root of expected local generalization}\bigg]\label{one_round-ap},
    \end{align}
    where $\delta_{k,\mathcal{A}}(\Sk[]) =  \bigg[R_{\sk}(\mathcal{A}(\Sk[])) - R_{\sk}(\mathcal{A}_k(\Sk))\bigg] $ indicates the level of non-iidness for client $k$ in function $\mathcal{A}$ on dataset $\Sk[]$.
\end{theorem}

\begin{proof}
We again consider $\Sk' = (\Z{k}{1}',...,\Z{k}{n_k}')$, where $\Z{k}{i}'$ is sampled from $\mathcal{D}_k$. 
Let also define $\Sk[]^{(k)} = (\Sk[1],...,\Sk',...,\Sk[K])$.
Based on Lemma \ref{LOO}, we can express the expected generalization error as
\begin{align}
    \EX_{\{\Sk[k] \sim \mathcal{D}_k^{n_k}\}_{k=1}^K} \Delta _{\mathcal{A}}(\Sk[]) = \EX_{k \sim \mathcal{K},\{\Sk[k], \Sk[k]' \sim \mathcal{D}_k^{n_k}\}_{k=1}^K}\bigg[\frac{1}{n_k} \sum_{i=1}^{n_k} \bigg(l(\mathcal{A}(\Sk[]),\Z{k}{i}') -  l(\mathcal{A}(\Sk[]^{(k)}),\Z{k}{i}')\bigg)\bigg].\label{eq0-ap}
\end{align}

Based on $L$-smoothness of $l(\Mk[{\Thetak[]}],\Zonly)$ in $\Mk[{\Thetak[]}]$, we obtain
\begin{align}
    \frac{1}{n_k} \sum_{i=1}^{n_k} \bigg(l(\mathcal{A}(\Sk[]),\Z{k}{i}') -  l(\mathcal{A}(\Sk[]^{(k)}),\Z{k}{i}')\bigg) \leq \inpr{\nabla \frac{1}{n_k} \sum_{i=1}^{n_k} l(\mathcal{A}(\Sk[]^{(k)}),\Z{k}{i}')}{\mathcal{A}(\Sk[]) - \mathcal{A}(\Sk[]^{(k)})}+ \frac{L}{2} \norm{\mathcal{A}(\Sk[]) - \mathcal{A}(\Sk[]^{(k)})},\label{eq1-ap}
\end{align}

where $\inpr{\cdot}{\cdot}, \norm{\cdot}$ indicate Euclidean inner product, and squared $L2$-norm.
Note that (\ref{eq1-ap}) holds due to
\begin{align}
    f(\vec{y}) \leq f(\vec{x}) + \inpr{\nabla f(\vec{x})}{\vec{y} - \vec{x}} + \frac{L}{2} \norm{\vec{y} - \vec{x}}.
\end{align}

We can bound expectation of the inner product term on the right-hand side of (\ref{eq1-ap}) using Cauchy--Schwarz inequality as
\begin{align}
     \EX& _{k \sim \mathcal{K},\{\Sk[k], \Sk[k]' \sim \mathcal{D}_k^{n_k}\}_{k=1}^K} \inpr{\nabla \frac{1}{n_k} \sum_{i=1}^{n_k} l(\mathcal{A}(\Sk[]^{(k)}),\Z{k}{i}')}{\mathcal{A}(\Sk[]) - \mathcal{A}(\Sk[]^{(k)})}\nonumber \\
     &\leq \EX _{k \sim \mathcal{K}}\EX _{\{\Sk[k], \Sk[k]' \sim \mathcal{D}_k^{n_k}\}_{k=1}^K} |\inpr{\nabla \frac{1}{n_k} \sum_{i=1}^{n_k} l(\mathcal{A}(\Sk[]^{(k)}),\Z{k}{i}')}{\mathcal{A}(\Sk[]) - \mathcal{A}(\Sk[]^{(k)})} |\label{eq2-0-ap}\\
     &\leq\EX _{k \sim \mathcal{K}}\bigg[\EX _{\{\Sk[k], \Sk[k]' \sim \mathcal{D}_k^{n_k}\}_{k=1}^K}||\nabla \frac{1}{n_k} \sum_{i=1}^{n_k} l(\mathcal{A}(\Sk[]^{(k)}),\Z{k}{i}')||\quad||\mathcal{A}(\Sk[]) - \mathcal{A}(\Sk[]^{(k)})|| \bigg] \label{eq2-1-ap}\\
     &\leq \EX _{k \sim \mathcal{K}}\sqrt{\EX_{\{\Sk[k], \Sk[k]' \sim \mathcal{D}_k^{n_k}\}_{k=1}^K}\norm{\nabla \frac{1}{n_k} \sum_{i=1}^{n_k} l(\mathcal{A}(\Sk[]^{(k)}),\Z{k}{i}')}\EX_{\{\Sk[k], \Sk[k]' \sim \mathcal{D}_k^{n_k}\}_{k=1}^K}\norm{\mathcal{A}(\Sk[]) - \mathcal{A}(\Sk[]^{(k)})}},\label{eq2-2-ap}
\end{align}
where (\ref{eq2-0-ap}) is true because on the right we have the absolute value. (\ref{eq2-1-ap}), and (\ref{eq2-2-ap}) are based on Cauchy--Schwarz inequality.

Now Let's find an upper bound for $\EX_{\{\Sk[k], \Sk[k]' \sim \mathcal{D}_k^{n_k}\}_{k=1}^K} \norm{\mathcal{A}(\Sk[]) - \mathcal{A}(\Sk[]^{(k)})}$ that appears on the right-hand side of both (\ref{eq1-ap}), and (\ref{eq2-2-ap}).
We obtain
\begin{align}
    \EX_{\{\Sk[k], \Sk[k]' \sim \mathcal{D}_k^{n_k}\}_{k=1}^K} &\norm{\mathcal{A}(\Sk[]) - \mathcal{A}(\Sk[]^{(k)})}\nonumber\\
    & = \EX_{\{\Sk[k], \Sk[k]' \sim \mathcal{D}_k^{n_k}\}_{k=1}^K} \mathcal{K}(k)^2 \norm{\mathcal{A}_k(\Sk) - \mathcal{A}_k(\Sk')} \label{eq4-1-ap} \\
    &\leq \EX_{\{\Sk[k], \Sk[k]' \sim \mathcal{D}_k^{n_k}\}_{k=1}^K} \frac{2\mathcal{K}(k)^2}{\mu} \bigg( R_{\Sk'}(\mathcal{A}_k(\Sk)) - R_{\Sk'}(\mathcal{A}_k(\Sk'))\bigg) \label{eq4-2-ap}\\
    & = \EX_{\{\Sk[k], \Sk[k]' \sim \mathcal{D}_k^{n_k}\}_{k=1}^K}\frac{2\mathcal{K}(k)^2}{\mu}\frac{1}{n_k} \sum_{j=1}^{n_k}\bigg( l(\mathcal{A}_k(\Sk),\Z{k}{j}') -  l(\mathcal{A}_k(\Sk'),\Z{k}{j}') \bigg)\label{eq4-3-ap}\\
    & = \EX_{\{\Sk[k], \Sk[k]' \sim \mathcal{D}_k^{n_k}\}_{k=1}^K}\frac{2\mathcal{K}(k)^2}{\mu} \Delta_{\mathcal{A}_k}(\Sk')\label{eq4-4-ap}\\
    & = \EX_{\{\Sk[k] \sim \mathcal{D}_k^{n_k}\}_{k=1}^K}\frac{2\mathcal{K}(k)^2}{\mu} \Delta_{\mathcal{A}_k}(\Sk)\label{eq4-5-ap},
\end{align}
where (\ref{eq4-1-ap}) proceeds by observing that $\mathcal{A}(\Sk[]^{(k,i)})$ varies solely in the sub-model derived from node $k$, diverging from ${\mathcal{A}(\Sk[])}$, and this discrepancy is magnified by a factor of $\mathcal{K}(k)$ when averaging of all sub-models.
(\ref{eq4-2-ap}) holds due to the $\mu$-strongly convexity of $l(\Mk[{\Thetak[]}],\Zonly)$ in $\Mk[{\Thetak[]}]$ which leads to $\mu$-strongly convexity of $R_{\sk}(\Mk[{\Thetak[]}])$ and the fact that $\mathcal{A}_k(\Sk')$ is derived from the local ERM, \ie $\mathcal{A}_k(\Sk')= \arg \min_{\Mk[]} \bigg(\sum_{i=1}^{n_k} l(\Mk[],\z{k}{i}')\bigg)$ and $\nabla R_{\Sk'}(\mathcal{A}_k(\Sk')) = 0$.
Note that if $f$ is $\mu$-strongly convex, we get
\begin{align}
        f(\vec{x}) - f(\vec{y}) + \frac{\mu}{2} \norm{\vec{x} - \vec{y}} \leq   \inpr{\nabla f(\vec{x})}{\vec{x} - \vec{y}}.
\end{align}
(\ref{eq4-3-ap}), (\ref{eq4-4-ap}) are based on local empirical and population risk definitions.

Now we bound $\EX_{k \sim \mathcal{K},\{\Sk[k], \Sk[k]' \sim \mathcal{D}_k^{n_k}\}_{k=1}^K}\norm{\nabla \frac{1}{n_k} \sum_{i=1}^{n_k} l(\mathcal{A}(\Sk[]^{(k)}),\Z{k}{i}')}$ on the right-hand side of (\ref{eq2-2-ap}). Note that
\begin{align}
    \norm{\nabla \frac{1}{n_k} \sum_{i=1}^{n_k} l(\mathcal{A}(\Sk[]^{(k)}),\Z{k}{i}')}& \leq 2L \bigg( \frac{1}{n_k} \sum_{i=1}^{n_k} l(\mathcal{A}(\Sk[]^{(k)}),\Z{k}{i}') - \frac{1}{n_k} \sum_{i=1}^{n_k} l(\mathcal{A}_k(\Sk'),\Z{k}{i}') \bigg) \label{eq3-1-ap}\\
    & \leq 2L \bigg(R_{\sk'}(\mathcal{A}(\Sk[]^{(k)})) - R_{\sk'}(\mathcal{A}_k(\Sk'))\bigg) \label{eq3-2-ap},
\end{align}
where \ref{eq3-1-ap} is obtained using the fact that for any $L$-smooth function $f$, we have
\begin{align}
\norm{\nabla f(\vec{x})} \leq 2L(f(\vec{x}) - f^*),
\end{align}
and the fact that $\mathcal{A}_k(\Sk')$ is derived from the local ERM, \ie $\mathcal{A}_k(\Sk')= \arg \min_{\Mk[]} \sum_{i=1}^{n_k} l(\Mk[],\z{k}{i}')$.
\ref{eq3-2-ap} is based on the definition of local empirical risk.

Putting (\ref{eq1-ap}) into (\ref{eq0-ap}) and considering (\ref{eq2-2-ap}) we get
    \begin{align}
        \EX&_{\{\Sk[k] \sim \mathcal{D}_k^{n_k}\}_{k=1}^K}\Delta_{\mathcal{A}}(\Sk[]) \nonumber \\
        &\leq \EX_{k \sim \mathcal{K}}\bigg[\EX_{\{\Sk[k], \Sk[k]' \sim \mathcal{D}_k^{n_k}\}_{k=1}^K}\frac{L}{2} \norm{\mathcal{A}(\Sk[]) - \mathcal{A}(\Sk[]^{(k)})} \\
        && \mathllap{+ \sqrt{\EX_{\{\Sk[k], \Sk[k]' \sim \mathcal{D}_k^{n_k}\}_{k=1}^K}\norm{\nabla \frac{1}{n_k} \sum_{i=1}^{n_k} l(\mathcal{A}(\Sk[]^{(k)}),\Z{k}{i}')}\EX_{\{\Sk[k], \Sk[k]' \sim \mathcal{D}_k^{n_k}\}_{k=1}^K}\norm{\mathcal{A}(\Sk[]) - \mathcal{A}(\Sk[]^{(k)})}}\bigg]}\nonumber\\
      &\leq \EX_{k \sim \mathcal{K}}\bigg[\EX_{\{\Sk[k] \sim \mathcal{D}_k^{n_k}\}_{k=1}^K}\frac{L\mathcal{K}(k)^2}{\mu} \Delta_{\mathcal{A}_k}(\Sk) \label{eq5-1-ap}\\
        && \mathllap{+ \sqrt{\EX_{\{\Sk[k], \Sk[k]' \sim \mathcal{D}_k^{n_k}\}_{k=1}^K}2L \bigg(R_{\sk'}(\mathcal{A}(\Sk[]^{(k)})) - R_{\sk'}(\mathcal{A}_k(\Sk'))\bigg) \EX_{\{\Sk[k] \sim \mathcal{D}_k^{n_k}\}_{k=1}^K}\frac{2\mathcal{K}(k)^2}{\mu} \Delta_{\mathcal{A}_k}(\Sk)}\bigg]}\nonumber\\
      &\leq \EX_{k \sim \mathcal{K}}\bigg[\EX_{\{\Sk[k] \sim \mathcal{D}_k^{n_k}\}_{k=1}^K}\frac{L\mathcal{K}(k)^2}{\mu} \Delta_{\mathcal{A}_k}(\Sk) \label{eq5-2-ap}\\
        && \mathllap{+ \sqrt{\EX_{\{\Sk[k] \sim \mathcal{D}_k^{n_k}\}_{k=1}^K}2L \bigg(R_{\sk}(\mathcal{A}(\Sk[])) - R_{\sk}(\mathcal{A}_k(\Sk))\bigg) \EX_{\{\Sk[k] \sim \mathcal{D}_k^{n_k}\}_{k=1}^K}\frac{2\mathcal{K}(k)^2}{\mu} \Delta_{\mathcal{A}_k}(\Sk)}\bigg]}\nonumber\\
     &\leq \EX_{k \sim \mathcal{K}}\bigg[\frac{L\mathcal{K}(k)^2}{\mu}\EX_{\{\Sk[k] \sim \mathcal{D}_k^{n_k}\}_{k=1}^K} \Delta_{\mathcal{A}_k}(\Sk) + 2\sqrt{\frac{L}{\mu}}\mathcal{K}(k)\sqrt{\EX_{\{\Sk[k] \sim \mathcal{D}_k^{n_k}\}_{k=1}^K} \delta_{k,\mathcal{A}}(\Sk[]) \EX_{\{\Sk[k] \sim \mathcal{D}_k^{n_k}\}_{k=1}^K}\Delta_{\mathcal{A}_k}(\Sk)}\bigg]\label{eq5-3-ap}\\,
    \end{align}
where in (\ref{eq5-1-ap}) we have applied (\ref{eq4-5-ap}), and (\ref{eq3-2-ap}).
(\ref{eq5-2-ap}) proceeds by considering that 
\begin{align}
    \EX_{\{\Sk[k], \Sk[k]' \sim \mathcal{D}_k^{n_k}\}_{k=1}^K}\bigg[R_{\sk'}(\mathcal{A}(\Sk[]^{(k)})) - R_{\sk'}(\mathcal{A}_k(\Sk'))\bigg]=\EX_{\{\Sk[k] \sim \mathcal{D}_k^{n_k}\}_{k=1}^K}\bigg[ R_{\sk}(\mathcal{A}(\Sk[])) - R_{\sk}(\mathcal{A}_k(\Sk)) \bigg].
\end{align}
In (\ref{eq5-3-ap}) we have used the definition of $\delta_{k,\mathcal{A}}(\Sk[]) =  \bigg[R_{\sk}(\mathcal{A}(\Sk[])) - R_{\sk}(\mathcal{A}_k(\Sk))\bigg] $.
This completes the proof and provides the following upper bound for $\EX_{\{\Sk[k] \sim \mathcal{D}_k^{n_k}\}_{k=1}^K}\Delta_{\mathcal{A}}(\Sk[])$,
\begin{align}
     \EX_{k \sim \mathcal{K}}\bigg[\frac{L\mathcal{K}(k)^2}{\mu}\EX_{\{\Sk[k] \sim \mathcal{D}_k^{n_k}\}} \Delta_{\mathcal{A}_k}(\Sk) + 2\sqrt{\frac{L}{\mu}}\mathcal{K}(k)\sqrt{\EX_{\{\Sk[k] \sim \mathcal{D}_k^{n_k}\}_{k=1}^K} \delta_{k,\mathcal{A}}(\Sk[]) \EX_{\{\Sk[k] \sim \mathcal{D}_k^{n_k}\}}\Delta_{\mathcal{A}_k}(\Sk)}\bigg].
\end{align}

\end{proof}

\section{Proof of Theorem \ref{mani_th}} \label{ap-b}
Here we provide an identical theorem as Theorem \ref{one-round_lemma_ap}, except that instead of ERM, multiple local stochastic gradient descent steps are used as the local optimizer.

\begin{theorem}
\label{one-round-sgd_lemma_ap}
    Let $l(\Mk[{\Thetak[]}],\Zonly)$ be $\mu$-strongly convex and $L$-smooth in $\Mk[{\Thetak[]}]$, $\Mk[{\Thetak[]}_k]=\mathcal{A}_k(\Sk)$ represents the model obtained by doing multiple local steps as in (\ref{local_sgd}) on local dataset $\Sk$, and  $\Mk[\hat{{\Thetak[]}}]=\mathcal{A}(\Sk[])$ is the model after one round of FedAvg ($\hat{{\Thetak[]}} = \EX_{k \sim \mathcal{K}} \Thetak$).
    Then, the expected generalization error, $\EX_{\{\Sk[k] \sim \mathcal{D}_k^{n_k}\}_{k=1}^K}\Delta_{\mathcal{A}}(\Sk[])$, is bounded by
\begin{align}
     \EX_{k \sim \mathcal{K}}\bigg[&\EX_{\{\Sk[k] \sim \mathcal{D}_k^{n_k}\}_{k=1}^K} \frac{2L\mathcal{K}(k)^2}{\mu} \bigg(\Delta_{\mathcal{A}_k}(\Sk) + 2\epsilon_k(\Sk) \bigg) \\
     &+ {\sqrt \frac{8L}{\mu}} \mathcal{K}(k)\sqrt{\EX_{\{\Sk[k] \sim \mathcal{D}_k^{n_k}\}_{k=1}^K} \bigg(\delta_{k,\mathcal{A}}(\Sk[])+\epsilon_k(\Sk) \bigg) \EX_{\{\Sk[k] \sim \mathcal{D}_k^{n_k}\}_{k=1}^K} \bigg(\Delta_{\mathcal{A}_k}(\Sk) + 2\epsilon_k(\Sk) \bigg)}\bigg]\nonumber,
\end{align}
    where $\epsilon_k(\Sk) =   R_{\sk}(\mathcal{A}_k(\Sk))- R_{\sk}(\mathcal{A}^*(\Sk))$.
\end{theorem}

\begin{proof}
All the steps are exactly the same as in the proof of theorem \ref{one-round_lemma_ap} except for the two steps below:

First, we use a new upper bound for $\EX_{\{\Sk[k], \Sk[k]' \sim \mathcal{D}_k^{n_k}\}_{k=1}^K} \norm{\mathcal{A}(\Sk[]) - \mathcal{A}(\Sk[]^{(k)})}$ that appears on the right-hand side of both (\ref{eq1-ap}), and (\ref{eq2-2-ap}).
We have
\begin{align}
    \EX&_{\{\Sk[k], \Sk[k]' \sim \mathcal{D}_k^{n_k}\}_{k=1}^K} \norm{\mathcal{A}(\Sk[]) - \mathcal{A}(\Sk[]^{(k)})}\nonumber\\
    & = \EX_{\{\Sk[k], \Sk[k]' \sim \mathcal{D}_k^{n_k}\}_{k=1}^K} \mathcal{K}(k)^2 \norm{\mathcal{A}_k(\Sk) - \mathcal{A}_k(\Sk')} \\
    & = \EX_{\{\Sk[k], \Sk[k]' \sim \mathcal{D}_k^{n_k}\}_{k=1}^K} \mathcal{K}(k)^2 \norm{\mathcal{A}_k(\Sk) - \mathcal{A}^*(\Sk') + \mathcal{A}^*(\Sk')  - \mathcal{A}_k(\Sk')} \label{eq40-ap} \\
    & = \EX_{\{\Sk[k], \Sk[k]' \sim \mathcal{D}_k^{n_k}\}_{k=1}^K} 2 \mathcal{K}(k)^2 \bigg(\norm{\mathcal{A}_k(\Sk) - \mathcal{A}^*(\Sk')} + \norm{\mathcal{A}^*(\Sk')  - \mathcal{A}_k(\Sk')}\bigg) \label{eq41-ap} \\
    & \leq  \EX_{\{\Sk[k], \Sk[k]' \sim \mathcal{D}_k^{n_k}\}_{k=1}^K} \frac{4\mathcal{K}(k)^2}{\mu} \bigg( R_{\Sk'}(\mathcal{A}_k(\Sk)) - R_{\Sk'}(\mathcal{A}^*(\Sk')) + R_{\Sk'}(\mathcal{A}_k(\Sk')) -  R_{\Sk'}(\mathcal{A}^*(\Sk'))\bigg) \label{eq42-ap} \\
    & = \EX_{\{\Sk[k], \Sk[k]' \sim \mathcal{D}_k^{n_k}\}_{k=1}^K} \frac{4\mathcal{K}(k)^2}{\mu} \bigg( R_{\Sk'}(\mathcal{A}_k(\Sk)) - R_{\Sk'}(\mathcal{A}_k(\Sk')) + 2 R_{\Sk'}(\mathcal{A}_k(\Sk')) -  2 R_{\Sk'}(\mathcal{A}^*(\Sk'))\bigg) \label{eq43-ap} \\
    & = \EX_{\{\Sk[k], \Sk[k]' \sim \mathcal{D}_k^{n_k}\}_{k=1}^K}\frac{4\mathcal{K}(k)^2}{\mu}\bigg(\frac{1}{n_k} \sum_{j=1}^{n_k}\big( l(\mathcal{A}_k(\Sk),\Z{k}{j}') -  l(\mathcal{A}_k(\Sk'),\Z{k}{j}') \big)+  2 R_{\Sk'}(\mathcal{A}_k(\Sk')) -  2 R_{\Sk'}(\mathcal{A}^*(\Sk'))\bigg)\\
    & = \EX_{\{\Sk[k], \Sk[k]' \sim \mathcal{D}_k^{n_k}\}_{k=1}^K}\frac{4\mathcal{K}(k)^2}{\mu} \bigg(\Delta_{\mathcal{A}_k}(\Sk') + 2 R_{\Sk'}(\mathcal{A}_k(\Sk')) -  2 R_{\Sk'}(\mathcal{A}^*(\Sk'))\bigg)\\
    & = \EX_{\{\Sk[k] \sim \mathcal{D}_k^{n_k}\}_{k=1}^K}\frac{4\mathcal{K}(k)^2}{\mu} \bigg(\Delta_{\mathcal{A}_k}(\Sk) + 2 R_{\Sk}(\mathcal{A}_k(\Sk)) -  2 R_{\Sk}(\mathcal{A}^*(\Sk))\bigg) \label{eq4-5-ap-new},
\end{align}
where in (\ref{eq40-ap}) $\mathcal{A}^*(\Sk)$ is the ERM on $\Sk$.
(\ref{eq41-ap}) is based on the following inequality.
\begin{align}
    \norm{\sum_{i=1}^n a_i} \leq n \sum_{i=1}^n \norm{a_i}.
\end{align}
% (\ref{eq44-ap}) proceeds by observing that in by strong convexity of $f$ we have
% \begin{align}
%         f(\vec{x}) - f^*  \leq  \frac{1}{2\mu} \norm{\nabla f(x)}.
% \end{align}

Secondly, we derive a new bound for $\EX_{k \sim \mathcal{K},\{\Sk[k], \Sk[k]' \sim \mathcal{D}_k^{n_k}\}_{k=1}^K}\norm{\nabla \frac{1}{n_k} \sum_{i=1}^{n_k} l(\mathcal{A}(\Sk[]^{(k)}),\Z{k}{i}')}$ on the right hand side of (\ref{eq2-2-ap}). We get
\begin{align}
    \norm{\nabla \frac{1}{n_k} \sum_{i=1}^{n_k} l(\mathcal{A}(\Sk[]^{(k)}),\Z{k}{i}')}& \leq 2L \bigg( \frac{1}{n_k} \sum_{i=1}^{n_k} l(\mathcal{A}(\Sk[]^{(k)}),\Z{k}{i}') - \frac{1}{n_k} \sum_{i=1}^{n_k} l(\mathcal{A}^*(\Sk'),\Z{k}{i}') \bigg)\\
    & \leq 2L \bigg(R_{\sk'}(\mathcal{A}(\Sk[]^{(k)})) - R_{\sk'}(\mathcal{A}^*(\Sk'))\bigg)\\
    & \leq 2L \bigg(R_{\sk'}(\mathcal{A}(\Sk[]^{(k)})) -R_{\sk'}(\mathcal{A}_k(\Sk'))+R_{\sk'}(\mathcal{A}_k(\Sk'))- R_{\sk'}(\mathcal{A}^*(\Sk'))\bigg)\\
    & \leq 2L \bigg(\delta_{k,\mathcal{A}}(\Sk[]^{(k)})+R_{\sk'}(\mathcal{A}_k(\Sk'))- R_{\sk'}(\mathcal{A}^*(\Sk'))\bigg) \label{eq3-1-ap-new}.
\end{align}

Let's define $\epsilon_k(\Sk) =   R_{\sk}(\mathcal{A}_k(\Sk))- R_{\sk}(\mathcal{A}^*(\Sk))$.
Putting (\ref{eq1-ap}) into (\ref{eq0-ap}) and considering (\ref{eq2-2-ap}) we get
    \begin{align}
        \EX_{\{\Sk[k] \sim \mathcal{D}_k^{n_k}\}_{k=1}^K}\Delta_{\mathcal{A}}(\Sk[])& \nonumber \leq \EX_{k \sim \mathcal{K}}\bigg[\EX_{\{\Sk[k], \Sk[k]' \sim \mathcal{D}_k^{n_k}\}_{k=1}^K}\frac{L}{2} \norm{\mathcal{A}(\Sk[]) - \mathcal{A}(\Sk[]^{(k)})} \\
        && \mathllap{+ \sqrt{\EX_{\{\Sk[k], \Sk[k]' \sim \mathcal{D}_k^{n_k}\}_{k=1}^K}\norm{\nabla \frac{1}{n_k} \sum_{i=1}^{n_k} l(\mathcal{A}(\Sk[]^{(k)}),\Z{k}{i}')}\EX_{\{\Sk[k], \Sk[k]' \sim \mathcal{D}_k^{n_k}\}_{k=1}^K}\norm{\mathcal{A}(\Sk[]) - \mathcal{A}(\Sk[]^{(k)})}}\bigg]}\nonumber\\
      &\leq \EX_{k \sim \mathcal{K}}\bigg[\EX_{\{\Sk[k] \sim \mathcal{D}_k^{n_k}\}_{k=1}^K} \frac{2L\mathcal{K}(k)^2}{\mu} \bigg(\Delta_{\mathcal{A}_k}(\Sk) + 2\epsilon_k(\Sk) \bigg) \label{eq5-1-ap-new}\\
        && \mathllap{+ \sqrt{\EX_{\{\Sk[k], \Sk[k]' \sim \mathcal{D}_k^{n_k}\}_{k=1}^K} 2L \bigg(\delta_{k,\mathcal{A}}(\Sk[]^{(k)})+\epsilon_k(\Sk') \bigg) \EX_{\{\Sk[k] \sim \mathcal{D}_k^{n_k}\}_{k=1}^K}\frac{4\mathcal{K}(k)^2}{\mu} \bigg(\Delta_{\mathcal{A}_k}(\Sk) + 2\epsilon_k(\Sk) \bigg)}\bigg]}\nonumber\\
    &\leq \EX_{k \sim \mathcal{K}}\bigg[\EX_{\{\Sk[k] \sim \mathcal{D}_k^{n_k}\}_{k=1}^K} \frac{2L\mathcal{K}(k)^2}{\mu} \bigg(\Delta_{\mathcal{A}_k}(\Sk) + 2\epsilon_k(\Sk) \bigg) \\
        && \mathllap{+ {\sqrt \frac{8L}{\mu}} \mathcal{K}(k)\sqrt{\EX_{\{\Sk[k] \sim \mathcal{D}_k^{n_k}\}_{k=1}^K} \bigg(\delta_{k,\mathcal{A}}(\Sk[])+\epsilon_k(\Sk) \bigg) \EX_{\{\Sk[k] \sim \mathcal{D}_k^{n_k}\}_{k=1}^K} \bigg(\Delta_{\mathcal{A}_k}(\Sk) + 2\epsilon_k(\Sk) \bigg)}\bigg]}\nonumber,
    \end{align}
where in (\ref{eq5-1-ap-new}) we have applied (\ref{eq4-5-ap-new}), and (\ref{eq3-1-ap-new}). 
This completes the proof.
\end{proof}

Now, we prove Theorem \ref{mani_th} as follows.
\begin{theorem}\label{mani_th_ap}
Let $l(\Mk[{\Thetak[]}],\Zonly)$ be $\mu$-strongly convex and $L$-smooth in $\Mk[{\Thetak[]}]$.
Local models at round $r$ are calculated by doing $\tau$ local gradient descent steps (\ref{local_sgd}), and the local gradient variance is bounded by $\sigma^2$, \ie $\EX_{\Zonly \sim \mathcal{D}_k} \norm{\nabla l(\Mk[{\Thetak[]}],\Zonly) - \EX_{\Zonly \sim \mathcal{D}_k}\nabla l(\Mk[{\Thetak[]}],\Zonly)}\leq \sigma^2$.
The aggregated model at round $r$, $\Mk[\hat{{\Thetak[]}}_r]$, is obtained by performing FedAvg, and the data points used in round $r$ (\ie $Z_{k,r}$) are sampled without replacement.
The average generalization error bound is
\begin{align}
    \EX_{\Sk[]} \Delta_{FedAvg}(\Sk[]) \hspace{-2pt}\leq \frac{1}{R} \sum_{r=1}^R &\EX_{k \sim \mathcal{K}}\bigg[\frac{2L\mathcal{K}(k)^2}{\mu} A  \label{main-ap}+ \sqrt{\frac{8L}{\mu}}\mathcal{K}(k)(AB)^{\frac{1}{2}}\bigg]\nonumber
\end{align}
where $A = \Tilde{O}\bigg(\sqrt{\frac{\mathcal{C}(\Mk[{\Thetak[]}])}{|Z_{k,r}|}} + \frac{\sigma^2}{\mu \tau} \bigg)$, $B= \Tilde{O}\bigg( \EX_{\{Z_{k,r}\}_{k=1}^K} \delta_{k,\mathcal{A}}(\{Z_{k,r}\}_{k=1}^K) + \frac{\sigma^2}{\mu \tau} \bigg)$, $\Tilde{O}$ hides constants and poly-logarithmic factors,  and $\mathcal{C}(\Mk[{\Thetak[]}])$  shows the complexity of the model class of $\Mk[{\Thetak[]}]$.
\end{theorem}

\begin{proof}
Based on the definition, we have
\begin{align}
    \EX&_{\{\Sk[k] \sim \mathcal{D}_k^{n_k}\}_{k=1}^K} \Delta_{FedAvg}(\Sk[]) = \EX_{\{\Sk[k] \sim \mathcal{D}_k^{n_k}\}_{k=1}^K} \frac{1}{R} \sum_{r=1}^R \EX_{k \sim \mathcal{K}} \bigg[\EX_{\Zonly \sim \mathcal{D}_k} l(\Mk[\hat{\thetak[]}_r],\Zonly) - \frac{1}{|Z_{k,r}|} \sum_{i \in Z_{k,r}}^K l(\Mk[\hat{\thetak[]}_r],\z{k}{i})\bigg]\\
    &= \frac{1}{R} \sum_{r=1}^R \EX_{ \{Z_{k,r} \sim \mathcal{D}_k^{|Z_{k,r}|}\}_{k=1}^K}  \EX_{k \sim \mathcal{K}} \bigg[\EX_{\Zonly \sim \mathcal{D}_k} l(\Mk[\hat{\thetak[]}_r],\Zonly)- \frac{1}{|Z_{k,r}|} \sum_{i \in Z_{k,r}}^K l(\Mk[\hat{\thetak[]}_r],\z{k}{i})\bigg]\\
    &= \frac{1}{R} \sum_{r=1}^R \EX_{ \{Z_{k,r} \sim \mathcal{D}_k^{|Z_{k,r}|}\}_{k=1}^K} \Delta _{\mathcal{A}}(\{Z_{k,r}\}_{k=1}^K) \label{4-3}\\
    &\leq \frac{1}{R} \sum_{r=1}^R \EX_{k \sim \mathcal{K}}\bigg[\EX_{\{Z_{k,r} \sim \mathcal{D}_k^{|Z_{k,r}|}\}} \frac{2L\mathcal{K}(k)^2}{\mu} \bigg(\Delta_{\mathcal{A}_k}(Z_{k,r}) + 2\epsilon_k(Z_{k,r}) \bigg) \label{4-4}\\
    && \mathllap{ + {\sqrt \frac{8L}{\mu}} \mathcal{K}(k)\sqrt{\EX_{\{Z_{k,r} \sim \mathcal{D}_k^{|Z_{k,r}|}\}_{k=1}^K} \bigg(\delta_{k,\mathcal{A}}(\{Z_{k,r}\}_{k=1}^K)+\epsilon_k(Z_{k,r}) \bigg) \EX_{\{Z_{k,r} \sim \mathcal{D}_k^{|Z_{k,r}|}\}} \bigg(\Delta_{\mathcal{A}_k}(Z_{k,r}) + \epsilon_k(Z_{k,r}) \bigg)}\bigg] }\nonumber\\
    &\leq \frac{1}{R} \sum_{r=1}^R \EX_{k \sim \mathcal{K}}\bigg[ \frac{2L\mathcal{K}(k)^2}{\mu} \Tilde{O} \bigg(\sqrt{\frac{\mathcal{C}(\Mk[{\Thetak[]}])}{|Z_{k,r}|}} + \frac{\sigma^2}{\mu \tau}  \bigg) \label{4-5}\\
    && \mathllap{ + {\sqrt \frac{8L}{\mu}} \mathcal{K}(k)\sqrt{ \Tilde{O}\bigg( \EX_{\{Z_{k,r} \sim \mathcal{D}_k^{|Z_{k,r}|}\}_{k=1}^K} \delta_{k,\mathcal{A}}(\{Z_{k,r}\}_{k=1}^K)+\frac{\sigma^2}{\mu \tau}  \bigg)  \Tilde{O} \bigg(\sqrt{\frac{\mathcal{C}(\Mk[{\Thetak[]}])}{|Z_{k,r}|}} +\frac{\sigma^2}{\mu \tau}  \bigg)}\bigg] }\nonumber,
\end{align}
where in (\ref{4-3}), $\mathcal{A}$ represents one-round FedAvg algorithm. In (\ref{4-4}) we have used Theorem \ref{one-round-sgd_lemma_ap}.
In (\ref{4-5}) we have used the conventional statistical learning theory originated with Leslie Valiant's probably approximately correct (PAC) framework \cite{Valiant1984ATO}.
We have also applied the optimization convergence rate bounds in the literature \cite{Stich2019TheEF}. Note that $\Tilde{O}$ hides constants and poly-logarithmic factors.
\end{proof}

\section{Partial Client Participation Setting}\label{partial-ap}

We first define an empirical risk for the partial participation distribution $\hat{\mathcal{K}}$ on dataset $\sk[]$, where $Supp(\hat{\mathcal{K}}) \neq \{1,\dots, K\}$ and $|Supp(\hat{\mathcal{K}})| = \hat{K} \leq K$,  as
\begin{align}
    R^{\hat{\mathcal{K}}}_{\sk[]}(\mk[{\thetak[]}]) &=  \EX_{k \sim {\hat{\mathcal{K}}}} R_{\sk}(\mk[{\thetak[]}]) =\EX_{k \sim {\hat{\mathcal{K}}}} \frac{1}{n_k} \sum_{i=1}^{n_k} l(\mk[{\thetak[]}],\z{k}{i}),
\end{align}
where ${\hat{\mathcal{K}}}$ is an arbitrary distribution on participating nodes that is a part of all nodes, and 
%
%
% In light of these notations, we can establish the following definition:
% \begin{align}
%     R_{\sk[]}(\mk[{\thetak[]}]) &=  \EX_{k \sim \mathcal{K}} R_{\sk}(\mk[{\thetak[]}]) \\ &=\EX_{k \sim \mathcal{K}} \frac{1}{n_k} \sum_{i=1}^{n_k} l(\mk[{\thetak[]}],\z{k}{i}),
% \end{align}
%  as the empirical risk on dataset $\sk[]$. $\mathcal{K}$ is an arbitrary distribution on nodes. 
% 
$R_{\sk}(\mk[{\thetak[]}])$ is the empirical risk for model $\mk[{\thetak[]}]$ on local dataset $\sk$. 
 We further define a partial population risk for model $\mk[{\thetak[]}]$ as
\begin{align}
     R^{\hat{\mathcal{K}}}(\mk[{\thetak[]}]) &= \EX_{k \sim {\hat{\mathcal{K}}}} R_k(\mk[{\thetak[]}]) = \EX_{k \sim {\hat{\mathcal{K}}},\Zonly \sim \mathcal{D}_k} l(\mk[{\thetak[]}],\Zonly),
 \end{align} where $R_k(\mk[{\thetak[]}])$ is the population risk on node $k$'s data distribution.
 
 Now, we can define the generalization error for dataset $\sk[]$ and function $\mathcal{A}(\sk[])$ as
 \begin{align}
     \Delta _{\mathcal{A}}(\sk[]) &= R(\mathcal{A}(\sk[])) - R^{\hat{\mathcal{K}}}_{\sk[]}(\mathcal{A}(\sk[]))\\
     &= \underbrace{R(\mathcal{A}(\sk[])) -  R^{\hat{\mathcal{K}}}(\mathcal{A}(\sk[]))}_\text{Participation gap} + \underbrace{R^{\hat{\mathcal{K}}}(\mathcal{A}(\sk[])) -  R^{\hat{\mathcal{K}}}_{\sk[]}(\mathcal{A}(\sk[]))}_\text{Out-of-sample gap}.\label{participation}
 \end{align}
The expected generalization error is expressed as $\EX_{{\Sk[k] \sim \mathcal{D}k^{n_k}}{k=1}^K} \Delta _{\mathcal{A}}(\Sk[])$. Note that the second term in (\ref{participation}), which is related to the difference between in-sample and out-of-sample loss, can be bounded in the same way as in Theorem \ref{one-round_lemma} and Theorem \ref{mani_th}.
The first term is associated with the participation of not all clients. In the following, we demonstrate that under certain conditions, this term would be zero in expectation.
 
We assume there is a meta-distribution $\mathcal{P}$ supported on all distributions $\hat{K}$.

\begin{lemma}\label{Scheme1}
    Let $\{x_i\}_{i=1}^K$ denote any fixed deterministic sequence. Assume $\mathcal{P}$ is derived by sampling $\hat{K}$ clients with replacement based on distribution $\mathcal{K}$ followed by an  equal probability on all sampled clients, \ie $\hat{\mathcal{K}}(k) = \frac{1}{\hat{K}}$.
    Then
    \begin{align}
        \EX_{\mathcal{P}} \EX_{k \sim {\hat{\mathcal{K}}}} x_k  = \EX_{k \sim \mathcal{K}} x_k, \EX_{\mathcal{P}} \EX_{k \sim {\hat{\mathcal{K}}}} \hat{\mathcal{K}}(k) x_k = \frac{1}{\hat{K}} \EX_{k \sim \mathcal{K}} x_k , \EX_{\mathcal{P}} \EX_{k \sim {\hat{\mathcal{K}}}} \hat{\mathcal{K}}^2(k) x_k = \frac{1}{\hat{K}^2} \EX_{k \sim \mathcal{K}} x_k.
    \end{align}
\end{lemma}
\begin{proof}
\begin{align}
    \EX_{\mathcal{P}} \EX_{k \sim {\hat{\mathcal{K}}}} x_k &= \EX_{\mathcal{P}} \frac{1}{\hat{K}} \sum_{i=1}^{\hat{K}}  x_i =  \frac{1}{\hat{K}} \sum_{i=1}^{\hat{K}} \EX_{\mathcal{P}} x_i= \EX_{\mathcal{P}} x_i = \EX_{k \sim \mathcal{K}} x_k\\
    \EX_{\mathcal{P}} \EX_{k \sim {\hat{\mathcal{K}}}} \hat{\mathcal{K}}(k) x_k &= \EX_{\mathcal{P}} \frac{1}{\hat{K}^2} \sum_{i=1}^{\hat{K}}  x_i =  \frac{1}{\hat{K}^2} \sum_{i=1}^{\hat{K}} \EX_{\mathcal{P}} x_i= \frac{1}{\hat{K}}\EX_{\mathcal{P}} x_i = \frac{1}{\hat{K}} \EX_{k \sim \mathcal{K}} x_k\\
    \EX_{\mathcal{P}} \EX_{k \sim {\hat{\mathcal{K}}}} \hat{\mathcal{K}}^2(k) x_k &= \EX_{\mathcal{P}} \frac{1}{\hat{K}^3} \sum_{i=1}^{\hat{K}}  x_i =  \frac{1}{\hat{K}^3} \sum_{i=1}^{\hat{K}} \EX_{\mathcal{P}} x_i= \frac{1}{\hat{K}^2}\EX_{\mathcal{P}} x_i = \frac{1}{\hat{K}^2} \EX_{k \sim \mathcal{K}} x_k
\end{align}   
\end{proof}

\begin{lemma}\label{Scheme2}
    Let $\{x_i\}_{i=1}^K$ denote any fixed deterministic sequence. Assume $\mathcal{P}$ is derived by sampling $\hat{K}$ clients without replacement uniformly at random followed by weighted probability on all sampled clients as $\hat{\mathcal{K}}(k) = \frac{\mathcal{K}(k)K}{\hat{K}}$.
    Then
    \begin{align}
    \EX_{\mathcal{P}} \EX_{k \sim {\hat{\mathcal{K}}}} x_k = \EX_{k \sim \mathcal{K}} x_k,
    \EX_{\mathcal{P}} \EX_{k \sim {\hat{\mathcal{K}}}} \hat{\mathcal{K}}(k) x_k = \frac{K}{\hat{K}}\EX_{k \sim \mathcal{K}} \mathcal{K}(k) x_k,
    \EX_{\mathcal{P}} \EX_{k \sim {\hat{\mathcal{K}}}} \hat{\mathcal{K}}^2(k) x_k = \frac{K^2}{\hat{K}^2}\EX_{k \sim \mathcal{K}} \mathcal{K}^2(k) x_k.
    \end{align}
\end{lemma}
\begin{proof}
\begin{align}
    \EX_{\mathcal{P}} \EX_{k \sim {\hat{\mathcal{K}}}} x_k &= \EX_{\mathcal{P}} \frac{K}{\hat{K}} \sum_{i=1}^{\hat{K}}  \mathcal{K}(i) x_i =  \frac{K}{\hat{K}} \sum_{i=1}^{\hat{K}}\EX_{\mathcal{P}}  \mathcal{K}(i) x_i = K \EX_{\mathcal{P}} \mathcal{K}(i) x_i = K \frac{1}{K} \sum_{i=1}^k \mathcal{K}(i) x_i = \EX_{k \sim \mathcal{K}} x_k\\
    \EX_{\mathcal{P}} \EX_{k \sim {\hat{\mathcal{K}}}} \hat{\mathcal{K}}(k) x_k &= \EX_{\mathcal{P}} \frac{K^2}{\hat{K}^2} \sum_{i=1}^{\hat{K}}  \mathcal{K}^2(i) x_i =  \frac{K^2}{\hat{K}^2} \sum_{i=1}^{\hat{K}}\EX_{\mathcal{P}}  \mathcal{K}^2(i) x_i = \frac{K^2}{\hat{K}} \EX_{\mathcal{P}} \mathcal{K}^2(i) x_i = \frac{K^2}{\hat{K}} \frac{1}{K} \sum_{i=1}^k \mathcal{K}^2(i) x_i \\
    && \mathllap{= \frac{K}{\hat{K}}\EX_{k \sim \mathcal{K}} \mathcal{K}(k) x_k}\nonumber\\
    \EX_{\mathcal{P}} \EX_{k \sim {\hat{\mathcal{K}}}} \hat{\mathcal{K}}^2(k) x_k &= \EX_{\mathcal{P}} \frac{K^3}{\hat{K}^3} \sum_{i=1}^{\hat{K}}  \mathcal{K}^3(i) x_i =  \frac{K^3}{\hat{K}^3} \sum_{i=1}^{\hat{K}}\EX_{\mathcal{P}}  \mathcal{K}^3(i) x_i = \frac{K^3}{\hat{K}^2} \EX_{\mathcal{P}} \mathcal{K}^3(i) x_i = \frac{K^3}{\hat{K}^2} \frac{1}{K} \sum_{i=1}^k \mathcal{K}^3(i) x_i \\
    && \mathllap{= \frac{K^2}{\hat{K}^2}\EX_{k \sim \mathcal{K}} \mathcal{K}^2(k) x_k}\nonumber
\end{align}  
\end{proof}

So based on lemmas \ref{Scheme1}, and \ref{Scheme2}, it becomes evident that the expectation of the participation gap in (\ref{participation}) becomes zero for both two methods, \ie
\begin{align}
    \EX_\mathcal{P} \bigg[R(\mathcal{A}(\sk[])) -  R^{\hat{\mathcal{K}}}(\mathcal{A}(\sk[]))\bigg] = 0.
\end{align}

The expected generalization error, $\EX_{\{\Sk[k] \sim \mathcal{D}_k^{n_k}\}_{k=1}^K}\Delta_{\mathcal{A}}(\Sk[])$,will be just the expectation of the second term in \ref{participation} that can be bounded using Lemma \ref{one-round_lemma_ap} by 
\begin{align}
     \EX_{k \sim \hat{\mathcal{K}}}\bigg[\frac{L\hat{\mathcal{K}}(k)^2}{\mu}\EX_{\{\Sk[k] \sim \mathcal{D}_k^{n_k}\}} \Delta_{\mathcal{A}_k}(\Sk) + 2\sqrt{\frac{L}{\mu}}\hat{\mathcal{K}}(k)\sqrt{\EX_{\{\Sk[k] \sim \mathcal{D}_k^{n_k}\}_{k=1}^K} \delta_{k,\mathcal{A}}(\Sk[]) \EX_{\{\Sk[k] \sim \mathcal{D}_k^{n_k}\}}\Delta_{\mathcal{A}_k}(\Sk)}\bigg]. \label{one-round-partial}
\end{align}

If we take the expectation of (\ref{one-round-partial}) with respect to $\mathcal{P}$, and taking into account Lemma \ref{Scheme1}, we get the generalization bound for method $1$ as 
\begin{align}
    \EX_{k \sim \mathcal{K}}\bigg[\frac{L}{\mu \hat{K}^2}\EX_{\{\Sk[k] \sim \mathcal{D}_k^{n_k}\}} \Delta_{\mathcal{A}_k}(\Sk) + 2\sqrt{\frac{L}{\mu}}\frac{1}{\hat{K}}\sqrt{\EX_{\{\Sk[k] \sim \mathcal{D}_k^{n_k}\}_{k=1}^K} \delta_{k,\mathcal{A}}(\Sk[]) \EX_{\{\Sk[k] \sim \mathcal{D}_k^{n_k}\}}\Delta_{\mathcal{A}_k}(\Sk)}\bigg].
\end{align}

For Scheme 2, we can obtain the generalization bound in the same way by taking the expectation of (\ref{one-round-partial}) with respect to $\mathcal{P}$ and considering Lemma \ref{Scheme2}. We get the generalization bound for Method 2 as:
\begin{align}
    \EX_{k \sim \mathcal{K}}\bigg[\frac{L}{\mu }\frac{\mathcal{K}(k)^2K^2}{\hat{K}^2}\EX_{\{\Sk[k] \sim \mathcal{D}_k^{n_k}\}} \Delta_{\mathcal{A}_k}(\Sk) + 2\sqrt{\frac{L}{\mu}}\frac{\mathcal{K}(k)K}{\hat{K}}\sqrt{\EX_{\{\Sk[k] \sim \mathcal{D}_k^{n_k}\}_{k=1}^K} \delta_{k,\mathcal{A}}(\Sk[]) \EX_{\{\Sk[k] \sim \mathcal{D}_k^{n_k}\}}\Delta_{\mathcal{A}_k}(\Sk)}\bigg].
\end{align}

\section{Detailed Experimental Setup} \label{detaild-exp}

\subsection{Image Classification}

The details are specified in Table \ref{table:imagenet_settings}.

\begin{table}[hbt]
\centering
\caption{Default experimental settings for the image classification training}
\label{table:imagenet_settings}
\renewcommand{\arraystretch}{1.2} % Adjust row spacing
\setlength{\tabcolsep}{6pt} % Adjust column spacing
\begin{tabular}{p{5cm} p{10cm}} % Adjusted column widths
\toprule
\textbf{Dataset} &  
CIFAR-10, CIFAR-100 \citep{cifar10}, \\
& SVHN \citep{Netzer2011ReadingDI}, \\
& MNIST \citep{mnist} \\ 
\midrule
\textbf{Architecture} & ResNet-20 \cite{He2015resnet} \\ 
\textbf{Training objective} & Cross entropy \\ 
\textbf{Test objective} & Top-1 accuracy \\ 
\midrule
\textbf{Number of clients} & 5 \\ 
\textbf{Data distribution} & IID (shuffled and split), \\
& Non-IID (sorted based on labels then split) \\ 
\textbf{Local Steps $\tau$} & 5 (unless explicitly specified)\\
\textbf{Adaptation coefficient $\alpha$} & 10 (unless explicitly specified) \\
\textbf{Representation Extractor} & Convolutional layers \\
& (unless explicitly specified) \\
\textbf{Head} & Final dense layers \\
& (unless explicitly specified) \\
\midrule
\textbf{Optimizer} & SGD with momentum \\
\textbf{Batch size} & 64 per client \\ 
\textbf{Momentum} & 0.9 (Nesterov) \\ 
\textbf{Learning rate} & Constant $\eta$ after tuning \\ 
\textbf{Number of Iterations} & $10^4$ for IID and $2 \times10^4$ for non-IID \\ 
\textbf{Weight decay} & $10^{-4}$ \\ 
\midrule
\textbf{Repetitions} & 20, with varying seeds \\ 
\textbf{Reported metric} & Mean and standard deviation of the aggregated model's test accuracies over the last 5 rounds \\ 
\bottomrule
\end{tabular}
\end{table}

\subsection{Large Language Model Fine-tuning}

The details are specified in Table \ref{table:llm_settings}.

\begin{table}[H]
\centering
\caption{Default experimental settings for the large language model fine-tuning}
\label{table:llm_settings}
\renewcommand{\arraystretch}{1.2} % Adjust row spacing
\setlength{\tabcolsep}{6pt} % Adjust column spacing
\begin{tabular}{p{5cm} p{10cm}} % Adjusted column widths
\toprule
\textbf{Dataset} &  
Multi-Genre Natural Language Inference (MultiNLI) corpus \\ 
\midrule
\textbf{Architecture} & OPT-125M \cite{He2015resnet} \\ 
\textbf{Training objective} & Cross entropy \\ 
\textbf{Test objective} & Top-1 accuracy \\ 
\midrule
\textbf{Number of clients} & 5 \\ 
\textbf{Data distribution} & IID (shuffled and split), \\
& Non-IID (sorted based on genre then split) \\ 
\textbf{Local Steps $\tau$} & 5 \\
\textbf{Adaptation coefficient $\alpha$} & 10 \\
\textbf{Representation Extractor} & First 10 attention layers \\
\textbf{Head} & Final 2 attention layers \\
\midrule
\textbf{Optimizer} & AdamW \\
\textbf{Batch size} & 16 sentences per client \\ 
\textbf{Adam $\beta_1$} & 0.9 \\
\textbf{Adam $\beta_2$} & 0.999 \\
\textbf{Adam $\epsilon$} & $10^{-8}$ \\
\textbf{Learning rate} & Decaying as $\frac{\eta}{\frac{t}{100}+10}$ where $t$ is the iteration number and $\eta$ is tuned \\ 
\textbf{Number of Iterations} & $10^4$ for IID and $2 \times10^4$ for non-IID \\ 
\textbf{Weight decay} & $10^{-4}$ \\ 
\midrule
\textbf{Repetitions} & 20, with varying seeds \\ 
\textbf{Reported metric} & Mean and standard deviation of the aggregated model's test accuracies over the last 5 rounds \\ 
\bottomrule
\end{tabular}
\end{table}

\vspace{5pt}
\subsection{Hyper-parameters and tuning details}
\vspace{15pt}

We independently tuned the learning rates for both image classification and large language model (LLM) fine-tuning experiments for each dataset, algorithm, and data heterogeneity scenario (both iid and non-iid). For image classification, the learning rate was tuned, while for LLM fine-tuning, we specifically tuned the initial learning rate. The tuning process followed these steps:

\begin{itemize}
    \item Begin with an initial value of 0.01 and perform the experiment several times, each time using a different random seed.
    \item Perform a grid search to find the optimal learning rate. Start by multiplying or dividing the initial value by powers of two. Test both larger and smaller learning rates, running the experiment multiple times with different random seeds. The best learning rate is the one that, on average, produces the best results. If this learning rate is between two others that give worse results, then you have found the optimal value.
\end{itemize}

% \vspace{5pt}
\subsection{Algorithms Used in the Experiments}

In this section, we provide a detailed list of our implementation of 
\begin{itemize}
    \item FedAvg (Algorithm \ref{alg:algorithm2})
    \item SCAFFOLD \cite{Karimireddy2019SCAFFOLDSC} (Algorithm \ref{alg:algorithm3})
    \item Combined method FedALS+SCAFFOLD (Algorithm \ref{alg:algorithm4})
\end{itemize}

Algorithm \ref{alg:algorithm2} demonstrates the implementation of FedAvg. This is a widely used algorithm in federated learning that performs model averaging after local updates from each client.

Moving on to Algorithm \ref{alg:algorithm3}, we observe that SCAFFOLD introduces an additional mechanism beyond just the model updates. Specifically, SCAFFOLD keeps track of a state that is specific to each client, referred to as the client control variate, denoted by $\ck[k,r]$. This control variate is a key feature of SCAFFOLD, which helps to correct the client drift by utilizing these client-specific states. It's important to note that the clients within SCAFFOLD maintain memory and persistently store the values of $\ck[k,r]$ and $\sum_{k=1}^K\ck[k,r]$. This stored information plays a crucial role in the algorithm's functionality. Additionally, it is worth recognizing that when the value of $\ck[k,r]$ consistently remains at zero, the SCAFFOLD algorithm essentially reduces to FedAvg. This highlights the close relationship between the two algorithms and emphasizes the special role of the control variates in differentiating SCAFFOLD from FedAvg.

Finally, Algorithm \ref{alg:algorithm4} showcases the integration of FedALS with SCAFFOLD, presenting a hybrid approach. It is important to observe that in this combined algorithm, the control variables are not handled in a uniform manner but are instead fragmented according to the different partitions of the model. This fragmentation occurs because there are distinct local step counts for different parts of the model, leading to a more complex update mechanism compared to the standard SCAFFOLD approach.

\begin{algorithm}[ht]
\caption{FedAvg}
\label{alg:algorithm2}
\textbf{Input}: Initial model $\{\thetak[k,1,0]\}_{k=1}^K$, Learning rate $\eta$, and number of local steps $\tau$.\\
\textbf{Output}: $\hat{{\thetak[]}}_R$
\begin{algorithmic}[1] %[1] enables line numbers
\FOR{Round $r$ in $1,...,R$}
\FOR{Node $k$ in $1,...,K$ \textbf{in parallel}}
\FOR{Local step $t$ in $0,...,\tau -1$}
\STATE Sample the batch $\mathcal{B}_{k,r,t}$ from $\mathcal{D}_k$.
\STATE $\thetak[{k,r,t+1}] \hspace{-2pt}= \thetak[{k,r,t}] - \frac{\eta}{|\mathcal{B}_{k,r,t}|} \sum_{i \in \mathcal{B}_{k,r,t}}\nabla l(\mk[{\thetak[{k,r,t}]}],\z{k}{i})$
\ENDFOR
\STATE $\thetak[{k,r+1,0}] = \frac{1}{K} \sum_{k=1}^K\thetak[{k,r,\tau}]$
\ENDFOR
\ENDFOR
\STATE \textbf{return} $\hat{{\thetak[]}}_R = \frac{1}{K} \sum_{k=1}^K\thetak[{k,R,\tau}]$
\end{algorithmic}
\end{algorithm}

\begin{algorithm}[ht]
\caption{SCAFFOLD}
\label{alg:algorithm3}
\textbf{Input}: Initial model $\{\thetak[k,1,0]\}_{k=1}^K$, Initial control variable $\{\ck[k,1]\}_{k=1}^K$, learning rate $\eta$, and number of local steps $\tau$.\\
\textbf{Output}: $\hat{{\thetak[]}}_R$
\begin{algorithmic}[1] %[1] enables line numbers
\FOR{Round $r$ in $1,...,R$}
\FOR{Node $k$ in $1,...,K$ \textbf{in parallel}}
\FOR{Local step $t$ in $0,...,\tau -1$}
\STATE Sample the batch $\mathcal{B}_{k,r,t}$ from $\mathcal{D}_k$.
\STATE $\thetak[{k,r,t+1}] \hspace{-2pt}= \thetak[{k,r,t}] - \eta \big( \frac{1}{|\mathcal{B}_{k,r,t}|}\sum_{i \in \mathcal{B}_{k,r,t}}\nabla l(\mk[{\thetak[{k,r,t}]}],\z{k}{i}) -\ck[k,r] + \frac{1}{K}\sum_{k=1}^K \ck[k,r] \big)$
\ENDFOR
\STATE $\ck[k,r+1] = \ck[k,r] - \frac{1}{K}\sum_{k=1}^K \ck[k,r] + \frac{1}{\eta \tau}(\thetak[{k,r,0}] - \thetak[{k,r,\tau}] )$
\STATE $\thetak[{k,r+1,0}] = \frac{1}{K} \sum_{k=1}^K\thetak[{k,r,\tau}]$
\ENDFOR
\ENDFOR
\STATE \textbf{return} $\hat{{\thetak[]}}_R = \frac{1}{K} \sum_{k=1}^K\thetak[{k,R,\tau}]$
\end{algorithmic}
\end{algorithm}

\begin{algorithm}[ht]
\caption{FedALS + SCAFFOLD}
\label{alg:algorithm4}
\textbf{Input}: Initial model $\{\thetak[k,1,0]=[\phik[k,1,0],\hk[k,1,0]]\}_{k=1}^K$, Initial control variable $\{\ck[k,1]=[\ck[{k,1}]^{\phik[]},\ck[{k,1}]^{\hk[]}]\}_{k=1}^K$, learning rate $\eta$, number of local steps for the head model $\tau$, adaptation coefficient $\alpha$.  \\
\textbf{Output}: $\hat{{\thetak[]}}_R$
\begin{algorithmic}[1] %[1] enables line numbers
\FOR{Round $r$ in $1,...,R$}
\FOR{Node $k$ in $1,...,K$ \textbf{in parallel}}
\FOR{Local step $t$ in $0,...,\tau -1$}
\STATE Sample the batch $\mathcal{B}_{k,r,t}$ from $\mathcal{D}_k$.
\STATE $\thetak[{k,r,t+1}] \hspace{-2pt}= \thetak[{k,r,t}] - \eta \big( \frac{1}{|\mathcal{B}_{k,r,t}|}\sum_{i \in \mathcal{B}_{k,r,t}}\nabla l(\mk[{\thetak[{k,r,t}]}],\z{k}{i}) -\ck[k,r] + \frac{1}{K}\sum_{k=1}^K \ck[k,r] \big)$
\IF {$\mod(r\tau + t , \tau)  = 0 $}
\STATE $\ck[{k,r}]^{\hk[]} \leftarrow \ck[k,r]^{\hk[]} - \frac{1}{K}\sum_{k=1}^K \ck[k,r]^{\hk[]} + \frac{1}{\eta \tau}(\hk[{k,r,0}] - \hk[{k,r,t}] )$
\STATE $\hk[{k,r,t}] \leftarrow \frac{1}{K} \sum_{k=1}^K\hk[{k,r,t}]$
\ENDIF
\IF {$\mod(r\tau + t ,  \alpha\tau)  = 0 $}
\STATE $\ck[{k,r}]^{\phik[]} \leftarrow \ck[k,r]^{\phik[]} -\frac{1}{K}\sum_{k=1}^K \ck[k,r]^{\phik[]} + \frac{1}{\eta \alpha\tau}(\phik[{k,{\lfloor \frac{r\tau + t - \alpha\tau }{\tau} \rfloor},\mod(r\tau + t - \alpha\tau ,\tau)}]^l - \phik[{k,r,t}]^l )$
\STATE $\phik[{k,r,t}] \leftarrow \frac{1}{K} \sum_{k=1}^K\phik[{k,r,t}]$
\ENDIF
\ENDFOR
\STATE $\ck[{k,r+1}] = \ck[{k,r}]$
\STATE $\thetak[{k,r+1,0}] = \thetak[{k,r,\tau}]$
\ENDFOR
\ENDFOR
\STATE \textbf{return} $\hat{{\thetak[]}}_R = \frac{1}{K} \sum_{k=1}^K\thetak[{k,R,\tau}]$
\end{algorithmic}
\end{algorithm}

\end{document}